\documentclass[11pt]{article}
\usepackage[utf8]{inputenc}
\usepackage[margin=1in]{geometry}
\usepackage[T1]{fontenc} %
\usepackage{graphicx}
\usepackage{mathpazo}
\usepackage{hyperref}
\hypersetup{
  colorlinks = true,  
  urlcolor = {blueGrotto},
  linkcolor = {royalBlue},
  citecolor = {navyBlue}
}
\usepackage{setspace}
\usepackage{booktabs} 
\usepackage{multirow} 
\usepackage{multicol}

\usepackage{xcolor}
\definecolor{niceRed}{RGB}{190,38,38}
\definecolor{blueGrotto}{HTML}{059DC0}
\definecolor{royalBlue}{HTML}{057DCD}
\definecolor{navyBlue}{HTML}{0B579C}
\definecolor{limeGreen}{HTML}{81B622}
\definecolor{nicePurple}{HTML}{9c27b0}
\definecolor{lightRoyalBlue}{HTML}{def2ff} 
\definecolor{gold}{HTML}{ffa300}

\usepackage[utf8]{inputenc} %
\usepackage[T1]{fontenc}    %
\usepackage{hyperref}       %
\usepackage{url}            %
\usepackage{booktabs}       %
\usepackage{amsfonts}       %
\usepackage{nicefrac}       %
\usepackage{microtype}      %
\usepackage{xcolor}         %

\usepackage[
  backref=true,
  backend=biber,
  natbib=true,
  style=alphabetic,
  sorting=alphabeticlabel,
  sortcites=true,
  minbibnames=3,
  maxbibnames=999,
  mincitenames=1,
  maxcitenames=3,
  minalphanames=3,
  maxalphanames=4,
  doi=false, url=false, eprint=false, isbn=false, 
]{biblatex}

\DeclareSortingTemplate{alphabeticlabel}{
  \sort[final]{%
    \field{labelalpha} %
  }
  \sort{%
    \field{year}   %
  }
  \sort{%
    \field{title}  %
  }
}

\AtBeginRefsection{\GenRefcontextData{sorting=ynt}}
\AtEveryCite{\localrefcontext[sorting=ynt]}
\usepackage{xcolor}
\definecolor{niceRed}{RGB}{190,38,38}
\definecolor{blueGrotto}{HTML}{059DC0}
\definecolor{royalBlue}{HTML}{057DCD}
\definecolor{navyBlue}{HTML}{0B579C}
\definecolor{limeGreen}{HTML}{81B622}
\definecolor{nicePurple}{HTML}{9c27b0}
\definecolor{lightRoyalBlue}{HTML}{def2ff} 
\definecolor{gold}{HTML}{ffa300}

\usepackage[utf8]{inputenc} %
\usepackage[T1]{fontenc}    %
\usepackage{hyperref}       %
\usepackage{url}            %
\usepackage{booktabs}       %
\usepackage{amsfonts}       %
\usepackage{nicefrac}       %
\usepackage{microtype}      %
\usepackage{xcolor}         %

\usepackage{microtype}
\usepackage{booktabs} %
 
\usepackage{hyperref}

\usepackage{algpseudocode} 
\usepackage[linesnumbered,ruled,vlined]{algorithm2e}
\SetKwInOut{Oracles}{Query Access}

\usepackage{amsmath}
\usepackage{amssymb}
\usepackage{mathtools}
\usepackage{amsthm}
\usepackage{bbm}
\usepackage[capitalize,noabbrev,nameinlink]{cleveref}
\usepackage{nicefrac}
\usepackage{thm-restate}

\usepackage[textsize=small]{todonotes}
\usepackage{mathrsfs}
\usepackage{mathtools} 
\usepackage{dsfont}
\usepackage[framemethod=TikZ]{mdframed}
\usepackage{pifont} %
\usepackage{color,multicol} 
\usepackage[inline]{enumitem} 
\usepackage{multirow}

\usepackage{centernot}

\usepackage[many]{tcolorbox} 

\usepackage[normalem]{ulem}

\usepackage{subcaption}   %
\usepackage{overpic,xcolor}
\newcommand{\subfigtag}[1]{(#1)} %

\theoremstyle{plain}
\newtheorem{inftheorem}{Informal Theorem}[section]
\newtheorem{theorem}{Theorem}[section]
\newtheorem{proposition}[theorem]{Proposition}
\newtheorem{lemma}[theorem]{Lemma}
\newtheorem{corollary}[theorem]{Corollary} 
\newtheorem{definition}[theorem]{Definition}
\theoremstyle{remark}
\newtheorem{remark}[theorem]{Remark}
\newtheorem{example}[theorem]{Example}
\newtheorem{observation}[theorem]{Observation}

\usepackage[T1]{fontenc} %
\usepackage{xcolor}                 %
\definecolor{bubblegreen}{RGB}{190,38,38}
\definecolor{bubblegray}{RGB}{241,240,240}
 
\usepackage{mathrsfs,dsfont,bbm}    %
\usepackage{multicol,multirow}      %
\usepackage[inline]{enumitem}       %
\usepackage[many]{tcolorbox}        %
\usepackage[normalem]{ulem}         %

\crefname{section}{Section}{Sections}
\crefname{theorem}{Theorem}{Theorems}
\crefname{assumption}{Assumption}{Assumptions}
\crefname{lemma}{Lemma}{Lemmas}
\crefname{definition}{Definition}{Definitions}
\crefname{conjecture}{Conjecture}{Conjectures}
\crefname{corollary}{Corollary}{Corollaries}
\crefname{construction}{Construction}{Constructions}
\crefname{conjecture}{Conjecture}{Conjectures}
\crefname{claim}{Claim}{Claims}
\crefname{observation}{Observation}{Observations}
\crefname{proposition}{Proposition}{Propositions}
\crefname{fact}{Fact}{Facts}
\crefname{question}{Question}{Questions}
\crefname{problem}{Problem}{Problems}
\crefname{remark}{Remark}{Remarks}
\crefname{example}{Example}{Examples}
\crefname{equation}{Equation}{Equations}
\crefname{appendix}{Appendix}{Appendices}
\crefname{algorithm}{Algorithm}{Algorithms}
\crefname{model}{Model}{Models}
\crefname{figure}{Figure}{Figures}

\AfterEndEnvironment{definition}{\noindent\ignorespaces}
\AfterEndEnvironment{infdefinition}{\noindent\ignorespaces}
\AfterEndEnvironment{example}{\noindent\ignorespaces}
\AfterEndEnvironment{assumption}{\noindent\ignorespaces}
\AfterEndEnvironment{lemma}{\noindent\ignorespaces}
\AfterEndEnvironment{theorem}{\noindent\ignorespaces}
\AfterEndEnvironment{proposition}{\noindent\ignorespaces}
\AfterEndEnvironment{fact}{\noindent\ignorespaces}
\AfterEndEnvironment{question}{\noindent\ignorespaces}
\AfterEndEnvironment{corollary}{\noindent\ignorespaces}
\AfterEndEnvironment{model}{\noindent\ignorespaces}
\AfterEndEnvironment{remark}{\noindent\ignorespaces}
\AfterEndEnvironment{proof}{\noindent\ignorespaces}
\AfterEndEnvironment{fact}{\noindent\ignorespaces}
\AfterEndEnvironment{minftheorem}{\noindent\ignorespaces}
\AfterEndEnvironment{inftheorem}{\noindent\ignorespaces}
\AfterEndEnvironment{maintheorem}{\noindent\ignorespaces}
\AfterEndEnvironment{restatable}{\noindent\ignorespaces}

\newcommand{\Stackrel}[2]{\stackrel{\mathmakebox[\widthof{\ensuremath{#2}}]{#1}}{#2}}
\newcommand{\eat}[1]{}
\newcommand{\yesnum}{\addtocounter{equation}{1}\tag{\theequation}}

\makeatletter
\newcommand{\customlabel}[2]{%
\protected@write \@auxout {}{\string \newlabel {#1}{{#2}{\thepage}{#2}{#1}{}} }%
\hypertarget{#1}{}
}
\makeatother

\newcommand{\quadtext}[1]{\quad\text{#1}\quad}
\newcommand{\qquadtext}[1]{\qquad\text{#1}\qquad}
 
\newcommand{\quadand}{\quadtext{and}}
\newcommand{\qquadand}{\qquadtext{and}}

\newcommand{\qquadwhere}{\qquadtext{where}}

\def\abs#1{\left| #1 \right|}
\def\sabs#1{| #1 |}

\newcommand{\sinparen}[1]{(#1)}
\newcommand{\sinbrace}[1]{\{#1\}}

\newcommand{\inbrace}[1]{\left\{#1\right\}}

\newcommand{\inparen}[1]{\left(#1\right)}
\newcommand{\insquare}[1]{\left[#1\right]}

\newcommand{\N}{\mathbb{N}}

\newcommand{\Z}{\mathbb{Z}}

\newcommand{\E}{\operatornamewithlimits{\mathbb{E}}}

\newcommand\ind{\mathds{1}}

\newcommand{\nfrac}[2]{\nicefrac{#1}{#2}}

\renewcommand{\epsilon}{\varepsilon}

\makeatletter
\newcommand*{\tran}{{\mathpalette\@tran{}}}
\newcommand*{\@tran}[2]{\raisebox{\depth}{$\m@th#1\intercal$}}
\makeatother

\newcommand{\wh}[1]{\widehat{#1}}

\def\<{\langle}
\def\>{\rangle}

\DeclareMathAlphabet{\mathpzc}{OT1}{pzc}{m}{it}

\newcommand{\customcal}[1]{\euscr{#1}}

\newcommand{\cC}{\customcal{C}}

\newcommand{\cI}{\customcal{I}}

\newcommand{\cL}{\customcal{L}}

\newcommand{\cP}{\customcal{P}}

\newcommand{\cX}{\customcal{X}}

\newcommand{\er}{\mathrm{er}}

\DeclareMathAlphabet{\mathdutchcal}{U}{dutchcal}{m}{n}
\SetMathAlphabet{\mathdutchcal}{bold}{U}{dutchcal}{b}{n}
\DeclareMathAlphabet{\mathdutchbcal}{U}{dutchcal}{b}{n}
 
\DeclareMathAlphabet\urwscr{U}{urwchancal}{b}{n}%
\DeclareMathAlphabet\rsfscr{U}{rsfso}{m}{n}
\DeclareMathAlphabet\euscr{U}{eus}{m}{n}
\DeclareFontEncoding{LS2}{}{}
\DeclareFontSubstitution{LS2}{stix}{m}{n}
\DeclareMathAlphabet\stixcal{LS2}{stixcal}{m} {n}

\newcommand{\itparagraph}[1]{\medskip \noindent\textit{#1}~~}
\renewcommand{\paragraph}[1]{\bigskip  \noindent\textbf{#1}~~}

\newcommand{\ie}{\textit{i.e.}}
\newcommand{\eg}{\textit{e.g.}}

\newcommand{\supp}{\operatorname{supp}}

\newcommand{\iid}{i.i.d.}

\renewcommand{\d}{{\rm d}} 

        \usepackage{tikz}
\usetikzlibrary{patterns}

\newcommand{\algo}[1]{\mathpzc{#1}}
\newcommand{\generator}{\algo{G}}

\renewcommand{\hat}{\widehat}

\usepackage{array}
\newcolumntype{L}[1]{>{\raggedright\let\newline\\\arraybackslash\hspace{0pt}}m{#1}}
\newcolumntype{C}[1]{>{\centering\let\newline\\\arraybackslash\hspace{0pt}}m{#1}}
\newcolumntype{R}[1]{>{\raggedleft\let\newline\\\arraybackslash\hspace{0pt}}m{#1}}

\setlist[enumerate]{leftmargin=20pt}
\setlist[itemize]{leftmargin=20pt}

\usepackage{soul}

\newcommand{\negspacePost}{\vspace{0mm}}
\newcommand{\negspacePre}{\vspace{0mm}}
\newcommand{\negparaspace}{\vspace{0mm}}

\newmdenv[
    backgroundcolor=lightgray!10, %
    roundcorner=5pt,            %
    linecolor=black,             %
    linewidth=1pt,               %
    innertopmargin=5pt,         %
    innerbottommargin=0pt,      %
    innerleftmargin=10pt,        %
    innerrightmargin=10pt,       %
    skipabove=5pt,              %
    skipbelow=0pt               %
]{curvybox}

\title{
On Characterizations for Language Generation:\\
Interplay of Hallucinations, Breadth, and Stability %
}

\author{\begin{tabular}{C{4.8cm}C{4.8cm}C{4.8cm}}
        {\bf Alkis Kalavasis} 
            & {\bf Anay Mehrotra} 
                & {\bf Grigoris Velegkas}\\[2mm]
        {Yale University} 
            & {Yale University} 
                & {Yale University}\\[1mm]
        \mbox{\small\texttt{\href{mailto:alkis.kalavasis@yale.edu}{alkis.kalavasis@yale.edu}}} 
            & \mbox{\small\texttt{\href{mailto:anaymehrotra1@gmail.com}{anaymehrotra1@gmail.com}}}
                & \mbox{\small\texttt{\href{grigoris.velegkas@yale.edu}{grigoris.velegkas@yale.edu}}}
        \\
        \end{tabular}
}

\date{}

\begin{document}

\maketitle

\begin{abstract}
  We study language generation in the limit -- introduced by \citet[NeurIPS]{kleinberg2024language} -- building on classical works of \citet[Inf. Control]{gold1967language} and \citet[STOC]{angluin1979finding}.
  \cite{kleinberg2024language}'s main result is an algorithm for generating from any countable language collection in the limit. 
  While their algorithm eventually generates unseen strings from the target language $K$, it sacrifices coverage or \emph{breadth}, \ie{}, its ability to generate a rich set of strings.
  Recent work introduces different notions of breadth and explores when generation with breadth is possible, leaving a full characterization of these notions open. 
  Our first set of results settles this by characterizing generation for existing notions of breadth and their natural extensions. 
  Interestingly, our lower bounds are very flexible and hold for many performance metrics beyond breadth -- for instance, showing that, in general, it is impossible to train generators which achieve a higher perplexity or lower hallucination rate for $K$ compared to other languages.
  Next, we study language generation with breadth and stable generators -- algorithms that eventually stop changing after seeing an arbitrary but finite number of strings -- and prove unconditional lower bounds for such generators, strengthening the results of \cite[STOC]{kalavasis2025limitslanguagegenerationtradeoffs} and demonstrating that generation with many existing notions of breadth becomes equally hard, when stability is required.
  This gives a separation for generation with approximate breadth, between stable and unstable generators, highlighting the rich interplay between breadth, stability, and consistency in language generation.

\end{abstract}

\newpage
\addtocontents{toc}{\protect\setcounter{tocdepth}{2}}
{
    \setstretch{1.05}
    \tableofcontents
}
\newpage

\negspacePre{}
\negspacePre{}
\section{Introduction}  
        \negspacePost{}
        \negspacePost{}
    
    Language generation has a rich history in computer science, dating back to the seminal work of \cite{shannon1951redundancy}, culminating in today’s Large Language Models (LLMs) that have revolutionized natural language processing and, more broadly, machine learning (ML). %
    Although the problem at the core of generation -- generate new and unseen strings given a sequence of examples from a target language $K$ -- is easy to state, a theoretical understanding of why LLMs are able to produce coherent text remains elusive. 
    Recently, \citet{kleinberg2024language} formalized this problem under a simple yet elegant model of \emph{language generation in the limit:} given a stream of strings from an unknown target language $K$ (belonging to a known collection of languages $\cL=\inbrace{L_1,L_2,\dots}$), learn to generate new, previously unseen, strings also belonging to this target language.
    
    Their model is reminiscent of online learning \citep{littlestone1988learning}; there are two players, the generator and the adversary who play the following game: First, the adversary fixes {a target language} $K \in \cL$ and an enumeration of $K$.\footnotemark{} Then, at any round $n \geq 1$, it presents the $n$-th element $x_n$ of the enumeration to the generator. The generator, given the strings $S_n = \{x_1,\dots,x_n\}$ seen so far, outputs a new string $w_n \notin S_n$ -- its guess for an unseen string in $K$.  
    The generator wins the game if eventually it learns ``to generate from $K$.'' {Formally}, the generator $\generator$ is said to generate from $\cL$ in the limit if for all $K \in \cL$ and any enumeration of $K$, there is a finite time $n^\star$ such that, {for any subsequent round }
    $n \geq n^\star$, $w_n$ is an unseen element of $K$, \ie{}, $w_n \in K \setminus S_n$.

    \negparaspace{}

    This model has deep connections to the classical works of \citet{gold1967language} and \citet{angluin1979finding,angluin1980inductive}, which studied the problem of language \emph{identification} in the limit. 
    In the Gold--Angluin model, like the above model, an algorithm observes an adversarially chosen enumeration of strings from some unknown target language $K=L_{i^\star}$.
    The only difference is that in the Gold--Angluin model the goal is to eventually \emph{identify} the index $i^\star$ of the correct language, whereas in the Kleinberg--Mullainathan (KM) model the goal is the simpler task of \emph{generation} -- \ie{}, of outputting unseen strings from $K$. 

    \negparaspace{}

    \footnotetext{An enumeration of $K$ is an infinite sequence of elements (potentially including duplicates) which does not contain elements outside $K$, and for every element $x \in K$ there is some position $n_x \in \N$ where $x$ appears.}
    
    Language identification turns out to be hard for essentially all infinite collections of languages. Indeed, Angluin showed that it is intractable for most interesting language collections, including regular languages.
    Surprisingly, \cite{kleinberg2024language} proved, in stark contrast, that language generation is tractable for \emph{{all}} countable collections of languages.
    They provided an elegant algorithm that, given any stream of input strings from a target language $K$ in a countable collection $\cL = \inbrace{L_1, L_2, \ldots}$, generates a sequence of previously unseen strings such that beyond a finite time step, all generated strings belong to the target language $K$. 

    \negparaspace{}
     
    \paragraph{Main Questions.}
    The KM algorithm eventually stops \emph{hallucinating}, as it ceases outputting elements outside of $K$ after a finite time. However, this property comes at a cost: the KM algorithm sacrifices \emph{breadth} -- \ie{}, the ability to generate diverse strings from the target language.
    As the algorithm eliminates hallucinations,  it generates from an increasingly smaller subset of the target language, resembling mode collapse in generative adversarial networks \citep{arjovsky2017towards}.
    This observation raises a fundamental question, left open by \citet{kleinberg2024language}:
    \begin{mdframed}
        \begin{center}
          \textbf{Question \#1.}~~  \emph{Is the trade-off between consistency and breadth inherent for generation? In other words, must any algorithm that eventually generates only valid strings from the target language necessarily sacrifice the ability to generate a broad subset of the language?}
        \end{center}
    \end{mdframed}
    To formalize this question, recent work \citep{kalavasis2025limitslanguagegenerationtradeoffs,charikar2024exploringfacetslanguagegeneration} 
    relaxed the requirement that the learner outputs one element
    at a time and allowed it to output a whole set of elements.
    This also allows for the case where, at some finite point, one can stop training and generate a rich set of responses.
    {With this change,} \citet*{kalavasis2025limitslanguagegenerationtradeoffs} proposed three distinct notions of breadth and showed that, for a large family of generators, language generation with breadth is as hard as language identification. 
    Adding to this result, \cite{charikar2024exploringfacetslanguagegeneration} proved the impossibility of generation with breadth for a specific language collection {(with any generator)}. While these results suggest a fundamental tension between consistency and breadth, a complete characterization of when different notions of generation \mbox{with breadth are achievable remains open.}

    \negparaspace{}

    Another intriguing direction initiated by \citet{kalavasis2025limitslanguagegenerationtradeoffs}  concerns the stability of generators: a stable generator is one that eventually stops changing its ``support,'' \ie, the set of elements it outputs, %
    after seeing a finite number of {distinct} strings from the target language. 
    Stability is a central object in online learning {and} has already been studied in language identification \citep{gold1967language}.
    {\citet{kalavasis2025limitslanguagegenerationtradeoffs} studied generation under stability showing that certain notions of generation with breadth are ``hard'' to achieve if generators (from a specific family) are required to be stable, but largely left characterizing the effect of stability on generation with other notions of breadth and with other generators outside this family open.}
    \begin{mdframed}
        \textbf{Question \#2.}~~\emph{How does stability interplay with consistency and breadth in language generation?}
    \end{mdframed}

    \subsection{Our Contributions and Technical Novelty}
    Our work is centered around answering Questions 1 and 2 in the model of language generation in the limit \citep{kleinberg2024language}.
    {Next, we describe our main results and techniques.}
    
    \paragraph{Results for Question \#1.}~~ %
    There are many notions of breadth in the literature, all attempting to {quantify} how much of the target language is covered by a generator.
    Our first set of results provides a complete characterization of all notions of breadth proposed in prior work (\cref{sec:results,apx:landscape}).
    {In \cref{sec:results}}, we illustrate our results with two of the simplest notions of breadth: \emph{exact breadth} and \emph{approximate breadth} \citep{kalavasis2025limitslanguagegenerationtradeoffs}.
    {Exact breadth is the strongest notion, requiring that after sufficiently many examples, the learner must be able to generate all unseen elements of the target language $K$.}
    {Approximate breadth relaxes this condition, requiring generators to output all but \textit{finitely} many unseen elements of $K$ after seeing enough examples.} 
    For exact breadth, we show that:
    \begin{inftheorem}[see \cref{thm:exact-breadth}]
        A generator $\generator$ can generate from a collection $\cL$ with exact breadth in the limit if and only if $\cL$ is identifiable in the limit.
        \label{infthm:1}
    \end{inftheorem}
    {Thus,} collections $\cL$ {admitting} generators with exact breadth are exactly those that are identifiable in the Gold--Angluin model; {{they} have a combinatorial characterization due to} \cite{angluin1979finding} {that we call \textit{Angluin's condition}}  (see \Cref{def:angliun-criterion}).
    This result
    strengthens \citet{kalavasis2025limitslanguagegenerationtradeoffs}'s lower bound {which only applied to generators with specific properties; since our result applies to all generators without assumptions, it requires a fundamentally different proof approach.}

    The above is essentially a negative result {because} the classes $\cL$ satisfying Angluin's condition are {known to be} very limited \citep{kleinberg2024language}. 
    {A natural follow-up question is whether relaxing the requirement to approximate breadth, where the generator can miss finitely many elements, might overcome this limitation. For this question, we show:}

    \begin{inftheorem}[see \cref{thm:approximate-breadth}]
        A generator $\generator$ can generate from a  collection $\cL$ with approximate breadth in the limit if and only if $\cL$ satisfies \mbox{weak Angluin's Condition (\cref{def:weakAngluin}).}
        \label{infthm:2}
    \end{inftheorem}
    A few remarks are in order. 
    First, the ``weak Angluin's condition'' (\cref{def:weakAngluin}) is a novel relaxation of Angluin's classic condition (\cref{def:angliun-criterion}) that we introduce {in this work}. 
    {We prove that this requirement is strictly weaker than Angluin's original condition ({\cref{rem:separation-angluin-weak-angluin}}), establishing that approximate breadth is strictly easier to achieve than exact breadth. Nevertheless, the weak Angluin's condition remains highly restrictive -- it is not even satisfied by regular languages, which are far simpler than human languages. This demonstrates that the trade-off between consistency and breadth is inherent and largely unavoidable, even when we weaken our breadth requirement.}

\itparagraph{Technical Novelty.}
    {We view generation with exact or approximate breadth as special cases of generation \textit{properties} relative to the target language. 
    Other such properties might include having uniquely low perplexity or hallucination rate for the target language compared to other languages} (see \cref{rem:beyondBreadth:1}).
    {Our characterizations of generation with breadth rely on two novel abstractions and also have consequences for other properties:}
    {The first is the uniqueness criterion (\cref{def:uniqueness}) which informally states that if generator $G$ satisfies property $P$ for language $L$, it cannot satisfy $P$ for any different language $L'$.}
    We prove the following implications: %
    \begin{enumerate}[itemsep=-1pt,leftmargin=15pt]
        \item[$\triangleright$] Properties $P$ with uniqueness can only be achieved for collections satisfying Angluin's condition.
        \item[$\triangleright$] {Exact breadth (like some other notions of breadth) satisfies} uniqueness, {establishing the necessity direction} of \Cref{infthm:1}. 
        {Sufficiency} is {simpler}: if $\cL$ satisfies Angluin's condition, {we can identify the target language} and use its index to generate with exact breadth.
    \end{enumerate}
    However, approximate breadth (along with some other notions of breadth) does not satisfy uniqueness, and {requires our second abstraction, the \emph{finite non-uniqueness criterion} (\cref{def:finiteNonUniqueness}).}
    {Informally, this weaker condition requires that} if $\generator$ satisfies a property $P$ for $L$, then it can also satisfy $P$ for another language $L'$ only if $L$ and $L'$ differ on finitely many elements. 
    We show that:

    \begin{enumerate}[itemsep=-1pt]
        \item[$\triangleright$] Properties $P$ with finite non-uniqueness can only be achieved for collections satisfying the weak Angluin's condition.
        \item[$\triangleright$] {Approximate breadth  satisfies finite non-uniqueness, establishing the necessity direction of \Cref{infthm:2}. Unlike \Cref{infthm:1}, the sufficiency direction is also non-trivial: collections satisfying weak Angluin's condition are not necessarily identifiable, so we develope a novel algorithm achieving approximate breadth for any such collection.}
    \end{enumerate}
    The most technically intricate parts of these constructions are the lower bounds, which rely on careful \emph{diagonalization} arguments. To establish the upper bounds we present several algorithms that are inspired by the work of \citet{kleinberg2024language} and the seminal work of \citet{angluin1980inductive}. We elaborate on these techniques in \cref{sec:tech-overview}.
    In summary, these reductions are the main tools that enable us to characterize all existing notions of {breadth in the literature and resolve \textbf{Question \#1.}}
   
    \itparagraph{Implications for Statistical Settings.} 
    Using reductions from prior work, our characterizations extend to statistical settings where examples are drawn from distributions rather than chosen adversarially. We provide unconditional characterizations of generation with exact and approximate breadth in the stochastic model, extending the conditional characterizations of \citet{kalavasis2025limitslanguagegenerationtradeoffs} that were limited to a specific generator family (\cref{rem:statistical,apx:statistical}).

    \paragraph{Results for Question \#2.}~~
    Next, we investigate how generation with breadth is affected by stability, where generators eventually stop changing their support (\cref{def:stable-generators}), {as defined by \cite{kalavasis2025limitslanguagegenerationtradeoffs}}. 
    {Our results show that stability creates a unified landscape across notions of breadth:}
    \begin{inftheorem}[see \cref{thm:stable}]
        A \textbf{stable} generator $\generator$ can generate from a countable collection $\cL$ with exact/approximate breadth in the limit if and only if $\cL$ is identifiable in the limit.
    \end{inftheorem}
    This reveals a stark separation between stable and unstable generators, as certain notions that only require the weak Angluin's condition without stability now require the full condition with stability. 
    We also introduce further weaker notions of breadth and make significant progress in characterizing when they can be achieved under stability; {to simplify the exposition of the main ideas, we defer these additional results to} \cref{apx:landscape}.

    \itparagraph{Technical Novelty.} %
       Requiring stability introduces an important challenge: unlike breadth, which can be verified at specific steps $t$, stability requires examining the infinite future sequence of a generator's behavior. Even if a generator appears stable for arbitrarily many steps, we cannot confirm stability without seeing its entire infinite execution. This challenge in verification breaks our earlier lower bound techniques, making the proof significantly more difficult, \mbox{and necessitating novel ideas (\cref{sec:technicalOverview}).}

       Our results comprehensively map the landscape of language generation with breadth, pinpointing when various notions are achievable and revealing the interplay between consistency, stability, and different notions of breadth. Our abstractions also extend beyond breadth, establishing impossibility results for other desirable generation properties (\cref{rem:beyondBreadth:1,rem:beyondBreadth:2}).

    \negspacePre{} 
    \subsection{Related Work}\label{sec:related-work} 
        \negspacePost{}

        Our work directly builds on the framework of \citet{kleinberg2024language}, who introduced the model of language generation in the limit. Since then, a growing line of research has explored various aspects of language generation with and without breadth (\eg{}, \cite{li2024generationlenslearningtheory,kalavasis2025limitslanguagegenerationtradeoffs,charikar2024exploringfacetslanguagegeneration,raman2025generationnoisyexamples,peale2024}). 
        
        \smallskip
        
        \noindent\textbf{Language Generation with Breadth.} 
           {Our work builds upon} \citet{kalavasis2025limitslanguagegenerationtradeoffs,charikar2024exploringfacetslanguagegeneration} who study language generation with breadth.
            \citet{kalavasis2025limitslanguagegenerationtradeoffs} introduced three notions of breadth: exact, approximate, and unambiguous. 
            {They explored both \citet{kleinberg2024language}'s online setting and its statistical counterpart -- where the strings are sampled from a distribution instead of being adversarially generated.
            For specific generator family and each notion of breadth, they characterized which countable collections $\cL$ enable generation with breadth (for the last two notions, they also require stability).}
    \citet{charikar2024exploringfacetslanguagegeneration} introduced exhaustive generation, {another notion of breadth,} and provided an unconditional lower bound by constructing a specific language collection for which no algorithm can generate exhaustively. {(They also studied questions beyond breadth, discussed in \cref{appendix:relatedWorks}).}
    Our work unifies and extends both approaches by providing complete characterizations for all these notions of breadth without assumptions on the generator family that hold for all countable language collections.

\smallskip

\noindent\textbf{Independent and Concurrent Work.}
Independently of and concurrently to this work, the authors of \cite{charikar2024exploringfacetslanguagegeneration} updated their manuscript  to include a characterization of
exhaustive generation \cite[{Theorem 4}]{charikar2024exploringfacetslanguagegenerationV2} which is similar to our {result on approximate breadth (\cref{thm:approximate-breadth})}. %
Our work provides several additional contributions beyond this shared result, including characterizations of all existing notions of breadth (\cref{sec:results:characterizations}), lower bounds for abstract {properties of generation} --  extending beyond breadth (\cref{sec:results:characterizations,rem:beyondBreadth:1,rem:beyondBreadth:2}), characterizations for stable generation (\cref{sec:results:stability}), and \mbox{{characterizations for} the statistical setting (\cref{rem:statistical}).}

        \smallskip
        
        \noindent\textbf{Subsequent Work.} Two papers follow-up on our {work} to study more fine-grained notions of breadth. 
        \citet{peale2024} introduce ``representation,'' a weaker notion of breadth that requires the generator's outputs to proportionally represent certain groups of (elements in) the domain. 
        \citet{kleinberg2025densitymeasureslanguagegeneration} {weaken}  approximate breadth by allowing generators to miss infinitely many elements from the target language, {instead focusing on the output set's} ``density'' {in the} target language.
        Both of these works address natural follow-up questions raised by our {results} while being orthogonal.

        \smallskip
        
        \noindent\textbf{Other Directions {of Work} in Language Generation.} {Beyond breadth, recent work has explored other aspects of language generation. \citet*{li2024generationlenslearningtheory} studied language generation with uncountable collections and analyzed sample complexity for generation. \citet{raman2025generationnoisyexamples} investigated language generation in a model where an adversary can introduce errors in the inputs, developing a robust framework for noisy settings.
        \cite{karbasi2025impossibilityautomatedhallucinationdetection} explored the complexity of determining if a specific generator $\generator$ is hallucinating.}  
 
 \negspacePre{}
\section{Preliminaries}\label{sec:preliminaries}
    \negspacePost{}
    In this section, we present some background on language identification and generation in the limit. 

    \paragraph{Notation.}
    Let $\Sigma$ be a finite alphabet (\eg{}, $\{a, b, \ldots, z\}$) and $\Sigma^*$ the set of all finite-length strings formed by concatenating symbols from $\Sigma$. 
    We define a language $L$ as an infinite subset of $\Sigma^*$. 
    A countable collection of languages is denoted by $\cL = \{L_1, L_2, \ldots\}$. 
    We define a generating algorithm $\generator = (\generator_n)_{n \in \N}$ as a sequence of mappings $\generator_n\colon (\Sigma^*)^n \to 2^{\Sigma^*}$ parametrized by the input size $n$. {In words, the generator maps a finite training set to a (potentially infinite) set of elements.

    \paragraph{Language Generation in the Limit.}
    We now formally define language generation in the limit. %
    \begin{definition}[Language Generation in the Limit \citep{kleinberg2024language}]\label{def:consistentGeneration}
        Let $\cL =  \{L_1, L_2,\dots\}$ be a collection of languages, $\generator ~{ =\inparen{\generator_n}}$ be a generating algorithm, 
        and $K \in \cL$ be some target language.
                The algorithm  $\generator$ {is said to generate} from $K$ in the limit if, for all enumerations of $K$, there is some $n^* \in \N$ such that for all steps $n \geq n^*$, the algorithm’s output {$\generator_n(S_n)$} {is a subset of} $K \setminus S_n$, where $S_n$ are the first $n$ elements of the enumeration. The collection $\cL$
                allows for generation in the limit if there is an algorithm  $\generator$ that 
                 generates from $K$ in the limit for any $K \in \cL.$
    \end{definition}
    To gain some intuition about this definition, consider the collection $\cL=\{{\Z}, L_1, {L_{-1}}, L_2, {L_{-2}}, \dots\}$ of thresholds over integers where, for each $i\in \Z$, $L_i=\{i, i+1, i+2, \dots\}$. 
    Suppose the target language is some $K \in \cL$ and the adversary first enumerates string $x_1$. The generator can deduce that $K = L_z$ for some $z \leq x_1$, \ie{}, $K \in \{\Z, L_{x_1}, L_{x_1-1}, \dots\}$. 
    Since the intersection of all of these languages is non-empty and is a strict superset of the strings enumerated so far (namely, {the intersection is $\inbrace{x_1+1,x_1+2,\dots}$}), the generator can generate an element that is guaranteed to be in $K$: for instance, it is sufficient to output ${\inbrace{x_1+1}}$.
    More generally, after seeing strings $x_1, x_2, \dots, x_i$, the generator can output {a singleton containing} any integer larger than $\max_i x_i.$ %
    
    For the problem to be interesting, \citet{kleinberg2024language} assumed throughout that each language in the collection has infinite cardinality, \ie{}, $|L_i| = \infty$ for all $i$. 
    (Otherwise, $K \setminus S_n$ eventually becomes empty.) 
    {They} showed that language generation in the limit is possible for \emph{all} countable collections of languages --  starkly contrasting results in language identification, discussed next. The KM algorithm is a {key starting point for our algorithms,} \mbox{and we discuss it in \Cref{sec:technicalOverview}.}

    \paragraph{Language Identification in the Limit.}
    Language identification in the limit was introduced by \citet{gold1967language} and has, since, been widely studied in learning theory. {The model is slightly different from that of generation}: while generation only requires producing valid examples from the target language $K = L_{i^*}$, identification requires the learner to eventually determine the exact identity $i^*$ (index) of the target language in the collection. 
    Despite this seemingly minor difference, identification is dramatically harder than generation:
    indeed, generation is possible for any countable collection \citep{kleinberg2024language}, but identification is only possible for very limited collections \citep{angluin1979finding,angluin1980inductive}, which satisfy a certain structural property that we explain next. 
    A formal definition of language identification appears in \cref{apx:further-background} \mbox{but is not essential for understanding this paper.} 
    
    \paragraph{Angluin's Condition.}
    A key concept in our analysis is Angluin's condition – a structural property of language collections $\cL$ that characterizes identifiability: $\cL$ is identifiable if and only if it satisfies Angluin's condition \citep{angluin1980inductive}. 
    Informally, a collection satisfies Angluin's condition if for any language $L \in \cL$, there exists a finite subset $T_L$ (called a tell-tale set) that serves as a finite ``fingerprint'' allowing one to distinguish $L$ from any other language that contains $T_L$. 
\begin{restatable}[Angluin's Condition \citep{angluin1980inductive}]{definition}{angluinsCondition}\label{def:angliun-criterion}
            Fix a language collection $\cL = \{L_1, L_2, \dots\}$.
            The collection $\cL$ is said to satisfy Angluin's condition if for any index $i$, there is a tell-tale, \ie{}, a finite set of strings $T_i$ such that $T_i$ is a subset of $L_i$, \ie{}, $T_i\subseteq L_i$, and the following holds:
            \begin{center}
                \centering 
                For all $j\geq 1$, if $L_j\supseteq T_i$, then $L_j$ is not a proper subset of $L_i$.
            \end{center}
\end{restatable}
    {Roughly,} this condition ensures that after observing enough examples from the target language, one can rule out all incorrect {languages}. We refer to \Cref{fig:angluin} for a visualization of the condition.

        \begin{remark}[Representation of the Generators]\label{rem:succinct-rep}
        The astute reader might observe that the previous definitions allow for generating algorithms that output infinite-sized objects.
        However, all our generating algorithms have succinct representations and this allows for computable algorithms that sample (\ie{}, generate) a new element, enumerate the support of all generatable elements, and, given an element, decide whether it belongs to the support (\ie{}, whether it is part of the enumeration).
        On the other hand, our lower bounds {are stronger, they} hold for functions that might not be computable.
    \end{remark}

\negspacePre{}
\section{{Overview of} Results and Techniques}\label{sec:results} %
        \negspacePost{}

        In this section, we present our main results. We begin with two notions of generation with breadth from prior work, provide characterizations of generation with breadth (\cref{sec:results:characterizations}) and their implications (\cref{rem:statistical}), examine stable generation (\cref{sec:results:stability}), and overview our proof techniques (\cref{sec:technicalOverview}). While we focus on exact and approximate breadth in {this section}, our techniques extend to all existing notions and their natural combinations; we present these extensions in \cref{apx:landscape}.

    \paragraph{Notions of Breadth.}
    Recent works have introduced various notions of breadth, capturing different aspects of how generators cover a target language. 
    The first notion, \emph{exact breadth} (introduced by \citet{kalavasis2025limitslanguagegenerationtradeoffs} and studied by \citet{charikar2024exploringfacetslanguagegenerationV2}).
    Given samples $S$, a generator $\generator$ has exact breadth for $K$ if $\generator(S) = K\setminus S$, meaning it generates all unseen strings in $K$.
    \begin{definition}\label{def:exactBreadth}
        Generator $\generator$ has exact breadth for language $K$ given samples $S$ if $\generator(S) =  K\setminus S$.
    \end{definition}
    In words, language generation in the limit with exact breadth requires that, for any target language $K\in \cL$ and any enumeration of $K$, there is an $n^*\geq 1$, such that for all $n\geq n^*$, after seeing $n$ elements of the enumeration $S_n$, $\generator$ achieves exact breadth for language $K$.

    Recognizing that this is a strong requirement, \citet{kalavasis2025limitslanguagegenerationtradeoffs} also introduced a natural relaxation, approximate breadth, which allows the generator to miss a finite number of elements. 
    \begin{definition}\label{def:approxBreadth}
        Generator $\generator$ has approximate breadth for language $K$ given samples $S$ if $\generator(S)\subseteq K$ and $\abs{K\setminus \generator(S)} < \infty$.
    \end{definition} 
    Again, one can naturally define language generation in the limit with approximate breadth as above. 
    Next, we present our results for these two notions of breadth. We mention that we also characterize generation under all other notions of breadth introduced in prior work (see \cref{apx:landscape}).
    \negparaspace{}

    \negspacePre{}
    \subsection{Results on Generation with Breadth}\label{sec:results:characterizations} 
    \negspacePost{}
    Our first result characterizes language generation with exact breadth.

\begin{theorem}[Exact Breadth $\iff$ Angluin's Condition]\label{thm:exact-breadth}
For any countable collection of languages $\cL$, there is a generator $\generator=(\generator_n)$ that generates with exact breadth from $\cL$ in the limit if and only if $\cL$ satisfies Angluin's condition.
\end{theorem}
This result establishes that generation with exact breadth is as hard as language identification in the limit, which is a much more challenging problem than generation in the limit without breadth constraints. Our characterization generalizes previous work in several ways: it removes technical conditions on the generators needed by \citet{kalavasis2025limitslanguagegenerationtradeoffs} and extends the unconditional lower bound of \citet{charikar2024exploringfacetslanguagegeneration}, which only held for a specific language collection.

\paragraph{Generalization to Any ``Unique'' Property.}
One side of this result, the upper bound, is simple: at a high level, if Angluin's condition holds, then language identification is possible (\ie{}, one can find $i^\star$ such that $K=L_{i^\star}$), and then, one can generate with exact breadth by outputting the first unseen string from $K$.
(That said, there are some difficulties because we do not know when we have found $i^\star$, and we handle this in our proofs.)
The other side, the lower bound, is non-trivial and is actually a corollary of a much more general result concerning a property we call \emph{uniqueness}. 
\begin{definition}[Uniqueness]\label{def:uniqueness}
    A property $P$ of generation satisfies the uniqueness criterion for a collection $\cL$ if no generator $\generator$ can simultaneously satisfy that property for two different languages $L\neq L'$ in $\cL$, \ie{}, if $\generator$ has property $P$ for $L$, then it cannot have $P$ for $L'\neq L$ and vice versa.
\end{definition}
We prove the following lower bound for any property satisfying the uniqueness criterion.
\begin{theorem}[Lower bound with Uniqueness]\label{thm:uniqueness}
    Let $P$ be any property of generation that satisfies the uniqueness criterion. For a countable collection of languages $\cL$, there exists an algorithm that generates with property $P$ from $\cL$ in the limit only if $\cL$ satisfies Angluin's Condition.
\end{theorem}
To gain some intuition, note that exact breadth satisfies this uniqueness criterion: if a generator $\generator$ generates a language $L$ with exact breadth (\ie{}, $\generator{(S)}=L{\setminus S}$), then it necessarily cannot generate any other language $L'\neq L$ with exact breadth. 
In contrast, approximate breadth does not satisfy uniqueness: for collections containing languages $L_1 \subseteq L_2$ that differ on only finitely many elements, a generator with support $L_1$ can simultaneously generate with approximate breadth from $L_1$ and $L_2$.
Like exact breadth, other notions of breadth in the literature also satisfy the uniqueness condition and \cref{thm:uniqueness} is a powerful tool for proving lower bounds for such notions.
\begin{remark}[Implications Beyond Generation with Breadth]
    \label{rem:beyondBreadth:1}
    The theorem also has implications well beyond breadth.  
    Assume we require only that a generator’s evaluation metric --
    \eg{}, lower perplexity or hallucination rate -- is \emph{strictly} better on a target language $K$ than on every other $L\neq K$.
    Even this weaker ``metric-separation'' goal is attainable \emph{only}
    when the language collection $\cL$ satisfies Angluin’s condition;
    if not, then no generator can perform strictly better for the target $K$ than the rest.
    This fundamental limit applies regardless of the specific metric. %
\end{remark}

\negparaspace{}

\paragraph{Characterization of Approximate Breadth.}
    Next, we move to approximate breadth.
    Since approximate breadth does not satisfy the uniqueness criteria introduced above, we cannot show a lower bound for approximate breadth based on Angluin's condition, and need new ideas. 
    In fact, the reason why approximate breadth does not satisfy it hints towards the required relaxation that we need to impose on Angluin's condition: languages that differ on finitely many elements need to be treated differently from languages that differ on infinitely many elements.
    Motivated by this, we introduce a variant of Angluin's condition we call the weak Angluin's condition:
    
 \begin{restatable}[Weak Angluin's Condition]{definition}{weakAngluinsCondition}\label{def:suffCondition}\label{def:weakAngluin}
      Fix a language collection $\cL = \{L_1, L_2, \dots\}$.
      The collection $\cL$ is said to satisfy the weak Angluin's condition if for any index $i$, there is a tell-tale, \ie{}, a finite set of strings $T_i$ such that $T_i$ is a subset of $L_i$, \ie{}, $T_i\subseteq L_i$, and the following holds:
        \begin{flushleft}
            \quad For all $j\geq 1$ such that $L_j\supseteq T_i$, one of the following holds.
            \begin{itemize}[itemsep=-3pt]
                \item Either $L_j$ is not a proper subset of $L_i$; or 
    
            \item $L_j$ is a proper subset and misses finitely many elements of $L_i$, \ie{}, $\abs{L_i\setminus  L_j}<\infty$.
            \end{itemize} 
        \end{flushleft}
        The tell-tale oracle is a primitive that, given an index $i,$ outputs an enumeration of the set $T_i$.
    \end{restatable}
    For a visualization of this condition, we refer to \Cref{fig:angluin}.
    This condition relaxes Angluin's condition by allowing language $L_j$ containing the tell-tale set $T_i$ of language $L_i$ to be a proper subset of $L_i$ provided $L_j$ misses only finitely many elements (see \Cref{fig:angluin}). We remark that this is a \emph{strict} weakening of Angluin's condition (see {\cref{rem:separation-angluin-weak-angluin}}).
    
    \begin{figure}[!ht]
        \centering
        \vspace{-3mm}
        \begin{subfigure}[b]{0.28\linewidth}
    \begin{overpic}[width=\linewidth,
                    trim={3cm 1cm 4cm 1cm},clip]{figures/angluinCondition1.pdf}
      \put(2,2){\subfigtag{a}} %
    \end{overpic}
    \phantomcaption\label{fig:angluin:a}
  \end{subfigure} %
  \hspace{10mm}
  \begin{subfigure}[b]{0.385\linewidth}
    \begin{overpic}[width=\linewidth,
                    trim={2.5cm 2cm 1.6cm 0.5cm},clip]{figures/angluinCondition2.pdf}
      \put(2,2){\subfigtag{b}}
    \end{overpic}
    \phantomcaption\label{fig:angluin:b}
  \end{subfigure}
        \vspace{-4mm}
        \caption{
            \cref{fig:angluin:a} visualizes Angluin's condition: any language $L'$ containing language $L$'s tell-tale set $T_L$ cannot be a strict subset of $L$. Our weak Angluin's condition relaxes this by allowing an additional case (\cref{fig:angluin:b}): a language $L'$ containing $T_L$ can be a strict subset of $L$ provided $L'$ only misses finitely many elements of $L$ (\ie{}, $|L \setminus L'| < \infty$).
        }
        \vspace{-2mm}
        \label{fig:angluin}
    \end{figure}

    \noindent Our next result characterizes approximate breadth via the Weak Angluin's Condition.
    \begin{theorem}
        [Approximate Breadth $\iff$ Weak Angluin's Condition]\label{thm:approximate-breadth}
        For any countable collection of languages $\cL$, there is {a generator $\generator=(\generator_n)$} that generates with approximate breadth from $\cL$ in the limit if and only if $\cL$ satisfies the weak Angluin's condition (\Cref{def:suffCondition}).
    \end{theorem}
    Since approximate breadth is characterized by the weak Angluin's condition, which is strictly weaker than Angluin's condition, approximate breadth is a {strictly} weaker requirement than exact breadth.

    Unlike the characterization of exact breadth, the upper bound side of this result is not simple. This is because if a language collection $\cL$ satisfies the weak Angluin's condition, it may not be identifiable, and hence we need a different algorithm for generation that achieves approximate breadth.
    We design a new algorithm based on the weak Angluin's condition and overview it in \cref{sec:technicalOverview}.
    Like with characterization of exact-breadth, the {lower}-bound side of this argument is non-trivial and a corollary of a more general result concerning a property of \emph{finite non-uniqueness}.

    \paragraph{Generalization to Any “Finitely Non-Unique” Property.}
    Roughly speaking, finite non-uniqueness relaxes uniqueness by allowing properties that can hold for two languages $L$ and $L'$ simultaneously but only when $L$ and $L'$ differ on finitely many elements. 
\begin{definition}[Finite Non-Uniqueness]\label{def:finiteNonUniqueness}
    A property $P$ of generation satisfies the finite non-uniqueness criterion for a collection $\cL$ if no generator $\generator$ can simultaneously satisfy that property for two languages $L,L'\in \cL$ that differ in infinitely many elements (\ie{}, when $\abs{L\triangle L'}=\infty$), \ie{}, if $\generator$ has property $P$ for $L$ and $L'$ both, then $\abs{L'\triangle L} < \infty$.
\end{definition}
    To gain some intuition, note that approximate breadth satisfies this finite non-uniqueness criterion: if a generator generates with approximate breadth from two different languages $L$ and $L'$, then these languages can only differ on finitely many elements. This follows because the generator's support must be largely contained in both languages (with only finitely many elements missing), which is only possible when $|L \triangle L'| < \infty$.
    
    Our next result shows that achieving any property which satisfies finite non-uniqueness is already impossible for any collection that does not satisfy the weak Angluin's condition.
    \begin{theorem}[Lower bound with Finite Non-Uniqueness]
        \label{thm:finiteNonUniqueness}
        Let $P$ be any property of generation satisfying the finite non-uniqueness criterion. For a countable collection $\cL$, there exists an algorithm that generates with property $P$ from $\cL$ in the limit only if $\cL$ satisfies the weak Angluin's Condition.
    \end{theorem}
    This lets us characterize every breadth notion in the literature, including approximate breadth.
    \begin{remark}[Implications Beyond Generation with Breadth]
        \label{rem:beyondBreadth:2}

        The same reasoning as in \cref{rem:beyondBreadth:1} yields lower bounds for an even milder objective:
        achieving \emph{optimal} (rather than uniquely optimal) performance on $K$ together with finitely many other languages. %
        If $\cL$ fails the \emph{weak} Angluin condition, then no generator can attain the best possible perplexity -- or any analogous metric -- on a finite set of languages $\cL'\subseteq \cL$ (with $\abs{\cL'}<\infty$) which includes the target language $K$ (\ie{}, $K\in \cL'$).
        
    \end{remark}
     
    \begin{remark}[Implications for Statistical Setting]
        \label{rem:statistical}
        Using \citet{kalavasis2025limitslanguagegenerationtradeoffs}'s framework, our results extend to statistical settings where strings are sampled from distributions rather than adversarially chosen.  Concretely, we provide unconditional characterizations for generation with both exact and approximate breadth in stochastic models -- improving upon the earlier conditional results that applied only to a specific generator family \citep{kalavasis2025limitslanguagegenerationtradeoffs}. See \cref{apx:statistical} for details.
    \end{remark}

    \begin{remark}[Separation Between \Cref{def:angliun-criterion} and \Cref{def:suffCondition} \citep{charikar2024exploringfacetslanguagegeneration}]\label{rem:separation-angluin-weak-angluin}
        {We highlight that there is a separation
        between the collections of languages that satisfy
        \Cref{def:angliun-criterion} and \Cref{def:suffCondition}, 
        which is taken from \citep{charikar2024exploringfacetslanguagegeneration}. Let
        $\cX = \N$, $L_i = \N \setminus \inbrace{i}$, and
        $\cL = \inbrace{\N, L_1, L_2, \ldots}.$ 
        Then, $\cL$ does not satisfy \Cref{def:angliun-criterion} but satisfies \Cref{def:suffCondition}. Thus, \Cref{def:suffCondition} is a strictly weaker condition 
        than \Cref{def:angliun-criterion}.}
    \end{remark}
    
\negspacePre{}\negspacePre{}
\subsection{Results on Stable Generation with Breadth}\label{sec:results:stability}
\negspacePost{}
Our next set of results focuses on \emph{stable generators} -- those whose support eventually stops changing -- a requirement motivated by practical algorithms that converge to a model and by Gold's original work, which also required stability. Under stability, the landscape changes dramatically:
\begin{restatable}
[Stable Generating Algorithm \citep{kalavasis2025limitslanguagegenerationtradeoffs}]{definition}        {defStableGenerators}\label{def:stable-generators}
 A generating algorithm $\generator=(\generator_n)$ is stable for a language collection $\cL$ if for any target language $K \in \cL$ 
            and for any enumeration of $K,$ 
            there is some finite $n^* \in \N$ such that for all $n, n' \geq n^*,$
            it holds that
             $\generator_n(S_n) = \generator_{n'}(S_{n'}) $.
\end{restatable}
\vspace{-5mm}
\begin{theorem}[Characterization for Stable Generation]\label{thm:stable}
Fix any countable collection of languages $\cL$. {$\cL$ satisfies Angluin's condition if and only if one of the following two equivalent conditions hold}
\begin{itemize}[itemsep=-2pt]
    \item[$\triangleright$] There is a stable algorithm that generates with approximate breadth from $\cL$.
    \item[$\triangleright$] There is a stable algorithm that generates with exact breadth from $\cL$.
\end{itemize}
\end{theorem}
Hence, exact and approximate breadth are equivalent under stability, both requiring the (full) Angluin's condition -- contrasting with our earlier result where approximate breadth only requires the weak Angluin's condition. In fact, a stronger result holds: all notions of breadth proposed in prior work collapse to this same characterization under stability. In \cref{apx:stability}, we prove this and also present additional results that allow hallucinations and introduce weaker breadth notions.

\negspacePre{}
\subsection{Technical Overview}\label{sec:tech-overview}
\negspacePost{}
    \label{sec:technicalOverview} 
        We now outline our proof techniques and their novelty, beginning with our lower bound results.

         \negparaspace{}
       
       \paragraph{Lower Bounds.}
            Our goal is to show that if a collection $\cL$ lacks a certain property (\eg{}, Angluin's condition), then no generator can achieve the corresponding notion of breadth  (\eg{}, exact breadth)  for $\cL$. {The full proofs appear in \cref{apx:proofs:lowerBounds}.}
            First, we overview techniques in existing works.
            \begin{itemize}[leftmargin=12pt,itemsep=-1.25pt]
                \item[$\triangleright$] \textit{Technique I: Generator-Specific Bounds.}
                \cite{kalavasis2025limitslanguagegenerationtradeoffs}'s approach require generators satisfy a technical condition (\cref{apx:mop}) that, roughly, enables access to their ``support,'' or the set of their outputs, allowing a reduction from language identification to generation with breadth. This, however, fails for unconditional \mbox{lower bounds which make no assumptions on generators.}
                \item[$\triangleright$] \textit{Technique II: Diagonalization for Identification.}
                    For the related problem of language identification, the standard and only technique for proving unconditional lower bounds is \textit{diagonalization} (\eg{}, \cite{gold1967language}). 
                    At a high level, it constructs an algorithm-dependent enumeration of target language $K$ in phases: in the $i$-th phase, it enumerates $L_i$, and either the algorithm $A$ fails to identify $L_i$ or $A$ guesses the index as $i$, at which point the enumeration advances to phase $i+1$.
                    This creates a dilemma: either a phase continues indefinitely (causing infinitely many identification errors) or infinitely many phases occur (meaning $A$ misidentifies the language $K = L_\infty$ infinitely often).
            \item[$\triangleright$] \textit{Technique III: Collection-Specific Bounds.}
                \cite{charikar2024exploringfacetslanguagegeneration} adapted the above  diagonalization technique to prove generation with breadth is impossible for a specific ``hard'' collection $\cL^*$ -- yielding the first unconditional lower bound for generation with breadth, albeit one limited to just one collection. 
            \end{itemize}
            The complementary limitations of prior work raise a natural question: Can we prove lower bounds for language generation with breadth that simultaneously apply to all generators and for all collections for which generation with breadth is fundamentally impossible?

        \itparagraph{Idea 1: Universal Diagonalization.}  
              We generalize Charikar and Pabbaraju's diagonalization from a specific ``hard'' collection $\cL^*$ to \textit{all} collections violating Angluin's condition -- which is a tight result since \citet{kalavasis2025limitslanguagegenerationtradeoffs} give a generator with exact breadth for collections satisfying this condition.
            Here, the key insight is leveraging the structure of collections that violate Angluin's condition: Specifically, we set $L_\infty$ to be the language witnessing this violation, and index the remaining languages $L_1,L_2,\dots$ with finite subsets of $L_\infty$: for each finite $T\subseteq L_\infty$, $L_T$ is the language containing $T$ and satisfying $L_T\subsetneq L_\infty$ (guaranteed by the violation of Angluin's condition).

             \negparaspace{}\negparaspace{}

        \itparagraph{Idea 2: Weak Angluin's Condition.}  
            This approach fails for approximate breadth because a generator can simultaneously achieve approximate breadth for multiple languages. We address this by introducing the weak Angluin's condition, a relaxation of the original, and proving it enables diagonalization for approximate breadth. This lower bound is also tight: we provide a novel algorithm achieving approximate breadth for any collection satisfying this weaker condition.  
             
             \negparaspace{}\negparaspace{}

        \itparagraph{Challenge: Diagonalization against Stable Generators.} 
            While our previous (unconditional) lower bounds apply to stable generators, they do not yield tight characterizations for notions like approximate breadth. %
            The core issue is that unlike breadth -- which can be verified at specific steps $t$ -- verifying stability requires examining the generator's behavior over infinitely many steps. 
            As even if a generator is stable for many steps, we cannot confirm its stabile without seeing its future behaviour. %

         \negparaspace{}
         \negparaspace{}

        \itparagraph{Idea 3: Lazy Analysis of Diagonalization.} 
            To address this challenge, we introduce a ``lazy analysis'' of  diagonalization, loosely inspired by techniques in computational complexity  \citep{arora2009complexity}. 
            Unlike the standard analysis of diagonalization where the adversary forces the generator to make ``mistakes'' at the end of each phase, here the  adversary cannot force a mistake every round. 
            Instead, this lazy analysis uses the fact that after waiting for sufficiently many rounds, the generator ``exhausts all possibilities'' and must make a mistake.
            Proving this requires a sophisticated technical construction which shows that a generator must either be unstable or generate without approximate breadth infinitely often.
            We believe this technique is of independent interest and can have further applications in the analysis of natural \mbox{properties of stable generators beyond breadth.

        }

        \paragraph{Upper Bounds.}
            Our upper bounds construct algorithms for generation with (different notions of) breadth that work whenever collection $\cL$ satisfies properties like Angluin's condition. 
         For exact breadth, one already exists in prior work \citep{kalavasis2025limitslanguagegenerationtradeoffs}.
            Here, we focus on approximate breadth; we develop two algorithms for it with different access models of $\cL$: {one with unrestricted access and another with only membership access (ability to query ``is $w \in L_i$?'').}
            The membership-only algorithm is a novel adaptation of \citet{angluin1980inductive}'s seminal algorithm {that} is presented in \cref{sec:proofs:upperbounds:membership}. Here, we overview the simpler unrestricted-access algorithm.

             \negparaspace{}
              \negparaspace{}
            
        \itparagraph{KM24's Algorithm.} 
            \citet{kleinberg2024language}'s algorithm, in every round $t,$ creates a {chain} of \emph{critical languages} $\cC_1\supsetneq \cC_2 \supsetneq \dots \supsetneq \cC_t$  with the property that, for large enough $t$, the target language $K$ enters this chain and remains in it. 
            Now their algorithm is simple: it outputs (unseen strings from) the last critical language.
            Unfortunately, this algorithm loses breadth as $t$ increases, as it keeps generating from the last element of a constantly decreasing chain. %
            A more detailed presentation of this algorithm appears in \Cref{apx:km-algo}.

                    \negparaspace{}
                     \negparaspace{}

        \itparagraph{New Analysis of KM24's Algorithm.}
            If $\cL$ satisfies Angluin's condition, then
            \citet{kalavasis2025limitslanguagegenerationtradeoffs} have already shown that this algorithm achieves exact breadth. 
            To achieve approximate breadth, we show that when $\cL$ satisfies weak Angluin's condition, the last critical language, $\cC_t$, misses at most finitely many elements of $K$ (\ie{}, $\abs{K\setminus \cC_t}<\infty$) for large enough $t$.
            This shows that the above algorithm achieves approximate breadth for such $\cL$.
            This reveals an interesting best-of-three-worlds property: if $\cL$ satisfies Angluin's condition it achieves exact breadth, if it satisfies weak Angluin's condition it achieves approximate breadth, otherwise it achieves consistent generation. This is particularly appealing as these conditions might be challenging to verify given limited access to $\cL$.
            Finally, to obtain algorithms for other existing \mbox{notions of breadth, we use this as a building block (\cref{sec:proofs:upperbounds:exhaustive}).}

\section{Summary {of} Characterizations {for} Language Generation {with Breadth}}  
    \label{apx:landscape}
    In this section, we summarize our characterizations for language generation with all existing notions of breadth, with additional results for new notions presented in \cref{apx:stability}.

\paragraph{Outline.} We first define two additional notions of breadth from prior work (\cref{apx:landscape:notions}), completing all definitions of breadth in prior works, alongside exact breadth (\cref{def:exactBreadth}) and approximate breadth (\cref{def:approxBreadth}) from {\cref{sec:results}}. 
We then provide characterizations for each notion (\cref{apx:landscape:breadth}), extending \cref{thm:exact-breadth,thm:approximate-breadth}.
Finally, we examine stable generation with breadth (\cref{apx:landscape:stability}), extending \cref{thm:stable}, and consider settings allowing some hallucinations, both for unstable (\cref{apx:landscape:breadth:hallucinations}) and stable generators (\cref{apx:landscape:stability:hallucinations}).

\subsection{Remaining Notions of Breadth in Prior Work}

\label{apx:landscape:notions}

In this section, we introduce two additional notions of breadth, unambiguous generation and exhaustive generation, completing all definitions of breadth in prior works, alongside exact breadth (\cref{def:exactBreadth}) and approximate breadth (\cref{def:approxBreadth}) from {\cref{sec:results}}.

\paragraph{Unambiguous Generation.} This relaxation of exact breadth by \citet{kalavasis2025limitslanguagegenerationtradeoffs} allows hallucination (outputting strings outside target language $K$) provided the generator performs ``better'' for $K$ than for any other language in the collection.

\begin{definition}[Unambiguous Generation in the Limit \citep{kalavasis2025limitslanguagegenerationtradeoffs}] 

\label{def:unambiguous}

Generator $\generator$ unambiguously generates from language $K$ given samples $S$ if 

\begin{equation}
    \textstyle \abs{\generator(S) \triangle K} < \min_{L\in \cL\colon L\neq K} \abs{\generator(S) \triangle L}\,,
\end{equation}
where $A\triangle B\coloneqq \inparen{A\setminus B}\cup \inparen{B\setminus A}$ for sets $A$ and $B$.

\end{definition}
While unambiguous generation is seemingly weaker than exact breadth and incomparable to approximate breadth, our characterization (\cref{Main}) reveals that it is as hard to achieve as exact breadth.

\paragraph{Exhaustive Generation.} \cite{charikar2024exploringfacetslanguagegeneration} proposed exhaustive generation.\footnote{The definition in \citep{charikar2024exploringfacetslanguagegeneration} differs slightly from \citep{charikar2024exploringfacetslanguagegenerationV2}. We use the updated version, though {our techniques also show that} both {properties} are characterized by the same condition. {}} Their formulation treats generators as mappings from domain sequences to domain \emph{enumerations}. For $i, n \in \N$, let $\generator_n(i)$ be the $i$-th element in the enumeration output in round $n$.

\begin{definition}[Exhaustive Generation in the Limit \citep{charikar2024exploringfacetslanguagegenerationV2}]\label{def:ExhaustiveGeneration} 
Generator $\generator$ exhaustively generates from language $K$ in round $n$ if   
\begin{equation}
     \textstyle
     {\left|\bigcup_{i=1}^\infty \generator_n(i) \setminus K\right| < \infty}
     \quadand
     \ {S_n \cup \bigcup_{j = 1}^{n-1} \generator_j(1) \cup \bigcup_{i = 1}^\infty \generator_n(i)} \supseteq K \,, 
\end{equation}
where $S_n$ is the set of elements enumerated until round $n$.

\end{definition}
Exhaustive generation is strictly weaker than exact breadth but seems incomparable to approximate breadth: it permits finite hallucinations (which approximate breadth forbids) but requires covering $K$ using potentially all past outputs (which approximate breadth does not require).
Our characterization (\cref{Main}) reveals that it is as hard to achieve as approximate breadth.

    \subsection{Generation with Breadth (Extension of \cref{thm:exact-breadth,thm:approximate-breadth} and Proof Sketch)}\label{apx:landscape:breadth}
        Our next result characterizes generation with all four existing notions of breadth in the literature. 
        \begin{theorem}
            [Characterizations of Language Generation with Breadth]
            \label{Main}
            For any countable collection of languages $\cL$ the following hold:
            \begin{enumerate}
                \item The following are equivalent: 
                \begin{itemize}[itemsep=0pt]
                    \item[$\triangleright$] There is an algorithm that generates with (exact) breadth from $\cL$ in the limit.
                    \item[$\triangleright$] There is an algorithm that generates unambiguously from $\cL$ in the limit. 
                    \item[$\triangleright$] $\cL$ satisfies Angluin's condition (\Cref{def:angliun-criterion}).
                \end{itemize}
                \item The following are equivalent:
                \begin{itemize}[itemsep=0pt]
                    \item[$\triangleright$] There is an algorithm that generates with approximate breadth from $\cL$ in the limit.
                    \item[$\triangleright$] There is an algorithm that generates exhaustively from $\cL$ in the limit.
                    \item[$\triangleright$] $\cL$ satisfies the weak Angluin's condition (\Cref{def:weakAngluin}).
                \end{itemize}
            \end{enumerate}
            \end{theorem}
            This results generalizes \cref{thm:exact-breadth,thm:approximate-breadth}.
            Like \cref{thm:exact-breadth,thm:approximate-breadth}, this result is \emph{unconditional}, requiring no particular structure on the generator.
            Hence, it strengthens the conditional lower bounds of \citep{kalavasis2025limitslanguagegenerationtradeoffs}. 
            It also applies to all countable language collections, strengthening the collection-specific results of \citet{charikar2024exploringfacetslanguagegeneration}.
            \begin{proof}[Proof Sketch of \cref{Main}] We outline four key components: 
            \begin{itemize}[itemsep=0pt]
                \item \textit{Upper bound when $\cL$ satisfies Angluin's condition:} Since $\cL$ is identifiable (by Angluin's result), we can convert any identification algorithm to an exact generator, as established by \citet{kalavasis2025limitslanguagegenerationtradeoffs}. Unambiguous generation follows since it is weaker than exact breadth. 
                \item \textit{Lower bound when $\cL$ violates Angluin's condition:} In \cref{apx:proofs:lowerBounds:uniqueness}, we prove that properties satisfying uniqueness are unachievable for collections violating Angluin's condition (see \cref{sec:technicalOverview} for a discussion). Given this, the present result follows since exact breadth and unambiguous generation both satisfy uniqueness.
                \item \textit{Upper bound when $\cL$ satisfies weak Angluin's condition:} Since weak Angluin's condition is strictly weaker than Angluin's condition, $\cL$ is generally not identifiable and so we cannot use algorithms from the above upper bound.
                We present new algorithms for this case in \cref{apx:proofs:upperBounds}. 
                \item \textit{Lower bound when $\cL$ violates weak Angluin's condition:} 
                    In \cref{apx:proofs:lowerBounds:finitenonuniqueness}, we prove that properties satisfying finite non-uniqueness are unachievable for collections violating weak Angluin's condition (see \cref{sec:technicalOverview} for some discussion). 
                    The result follows from this since approximate breadth and exhaustive generation both satisfy finite non-uniqueness. 
            \end{itemize} 
            \end{proof}

    \subsection{Generation with Breadth and Stability (Extension of \cref{thm:stable} and Proof Sketch)}\label{apx:landscape:stability}
        In this section, we provide characterizations for generation with stable generators, those whose support eventually stops changing and stabilizes (\cref{def:stable-generators}).
        \begin{remark}[Discussion on Stability]
            This notion of stability stems from the original work of \cite{gold1967language} on language identification in the limit, where Gold requires the learner to stabilize to a specific guess for the target language $L$ in the above sense (see \cref{sec:furtherbackground}).
            It is also
            closely related to the question of whether the algorithm can  verify that it has ``learned'' to generate with the required notion of breadth; if the algorithm can verify that {it has learned}, then it can stabilize. 
            Further, any generator that is consistent and achieves exact breadth is also stable, since after some finite point its support must become identical to the target language $K$ and remain so.\footnote{Here, we use an equivalent notion of generation with exact breadth that allows for inclusion of the training set in the support: the equivalence holds because any generator $\generator$ that generates with breadth without repeating training examples can be converted to one $\generator'$ that generates with breadth and repeats the training examples and vice versa.}
        \end{remark}

        \paragraph{Landscape with Stable Generators.}
            Under the stability requirement, the landscape for generation with breadth changes (compared to the one in the previous section). 
            \begin{theorem}
            [Characterizations of Stable Language Generation with Breadth]
            \label{Main:stable}
            For any countable collection of languages $\cL$, the following are equivalent:
                \begin{itemize}[itemsep=-1pt]
                    \item There is a stable algorithm that generates with approximate breadth from $\cL$ in the limit.
                    \item There is a stable algorithm that generates exhaustively from $\cL$ in the limit.
                    \item There is a stable algorithm that generates with (exact) breadth from $\cL$ in the limit.
                    \item There is a stable algorithm that generates unambiguously from $\cL$ in the limit.
                     \item $\cL$ satisfies Angluin's condition (\Cref{def:angliun-criterion}).
                \end{itemize} 
            \end{theorem}
            This result extends \cref{thm:stable}.
            Like \cref{thm:stable}, it shows that the requirement of stability makes the problem of generation with approximate breadth strictly harder (see \cref{fig:conditionalLandscape}): 
            there exist stable generators with this property if and only if the collection satisfies Angluin's condition for identifiability whereas before, when unstable generators were also allowed,  one only required the weak Angluin's condition.
            As another example of the stark change in the landscape, we also show that there exists a collection that satisfies the weak Angluin's condition (hence admits a non-stable generator with approximate breadth), but for which no stable generator can achieve a much weaker requirement, which we term infinite coverage (\cref{thm:impossibility-stable-generation}).

             \begin{proof}[Proof Sketch of \cref{Main:stable}]
We outline two main components:
\begin{itemize}[itemsep=0pt,leftmargin=18pt]
    \item \textit{Upper bound when $\cL$ satisfies Angluin's condition:} This follows by observing that \cref{Main}'s upper bound for collections satisfying Angluin's condition constructs stable generators.
    
    \item \textit{Lower bound when $\cL$ violates Angluin's condition:} For exact breadth and unambiguous generation, this follows from \cref{Main}. The key technical challenge is establishing lower bounds for approximate breadth and exhaustive generation, requiring a certain ``lazy analysis'' of diagonalization as discussed in \cref{sec:technicalOverview}. 
    The proof appears in \cref{apx:proofs:lowerBounds:stability}.
\end{itemize}
\end{proof} 
            
\begin{figure}[th!]
    \centering
    \begin{subfigure}[t]{0.45\linewidth}
        \centering
        \includegraphics[width=\linewidth,trim={2.5cm 0cm 2.5cm 8cm},clip]{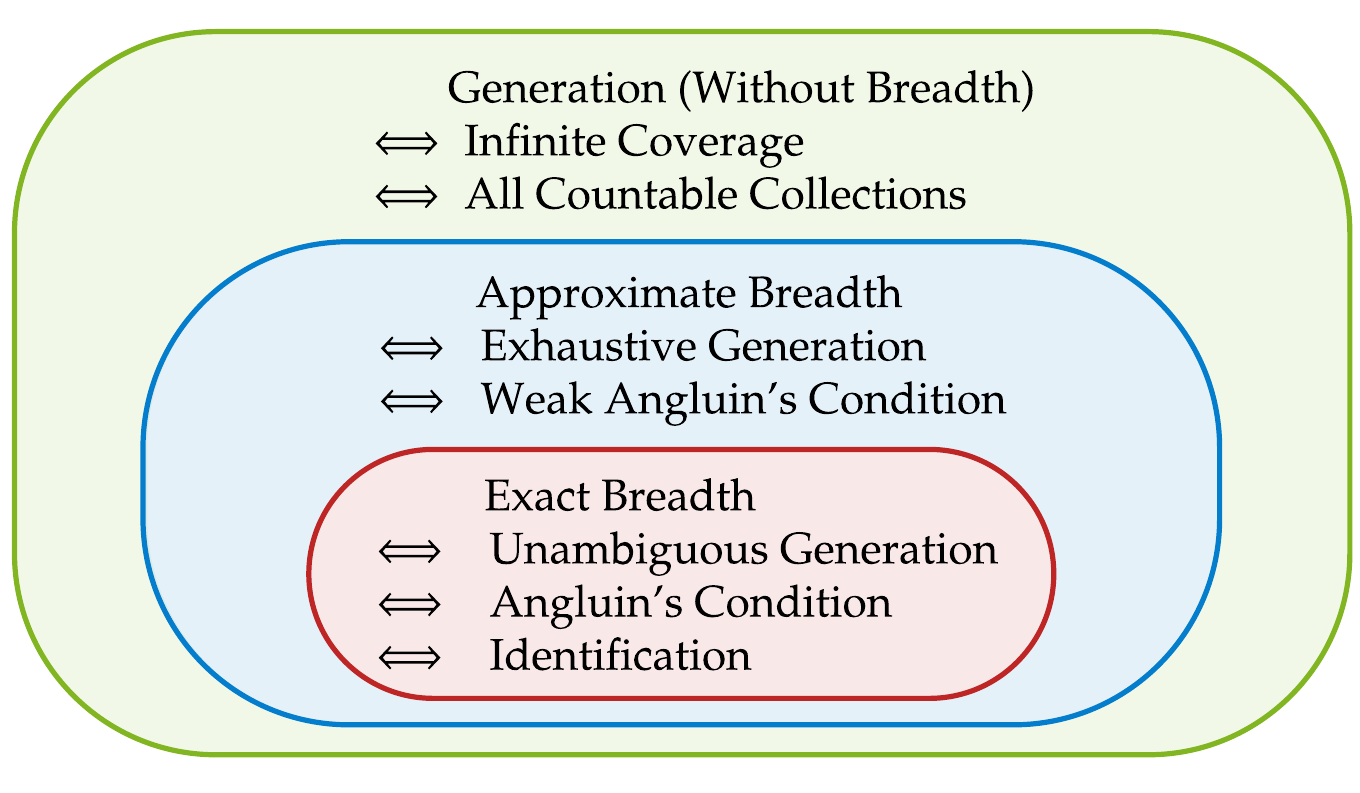}
        \caption{Unconditional Characterizations}
    \end{subfigure}
    \begin{subfigure}[t]{0.45\linewidth}
        \centering
        \includegraphics[width=\linewidth,trim={2.5cm 8.05cm 2.5cm 0cm},clip]{figures/stableVsUnstableGeneration.pdf}
        \caption{Characterization With Stable Generators}
    \end{subfigure}
    \caption{
        \textit{Comparison of Generation in the Limit with and without Requiring Stability.}\quad 
        Each containment illustrated by a border is \textit{strict}, \ie{}, for each border there is a language collection that satisfies the outer containment but not the inner containment.
        Concretely, in the figure on the left, there are (1) language collections that do not satisfy the Weak Angluin's Condition (\cref{def:weakAngluin}) (see \cref{ex:ar-progression}), (2) language collections that satisfy the Weak Angluin's Condition, but not Angluin's condition (see \cref{ex:weak-Angluin-no-angluin}), and (3) there are language collections which satisfy  Angluin's Condition (\cref{def:angliun-criterion}) (\eg{}, all finite collections).
        The figure on the right depicts the characterization for stable generators. 
        In addition to what is depicted there, there are (1) language collections that satisfy the weak Angluin's condition and for which infinite coverage is not achievable (see \cref{thm:impossibility-stable-generation}) and (2) language collections for which infinite coverage is achievable but that do not satisfy the weak Angluin's Condition (\cref{def:weakAngluin}) (see \cref{ex:ar-progression}). 
        We note that (1) and (2) \textit{are} not depicted in the right figure.
        }
    \label{fig:conditionalLandscape}
\end{figure}

\subsection{Generation with Breadth and Hallucinations}\label{apx:landscape:breadth:hallucinations}
    To illustrate the generality of our techniques, we use them to obtain characterizations for several new notions of generation.
    In particular, we obtain characterizations for generation with breadth where we relax the requirement that the generation becomes consistent (\ie{}, it has no hallucinations) in the limit.
    Instead, we allow for two cases:
    \begin{itemize}
        \item[$\triangleright$] \textit{Finite Hallucinations:}  Generator $\generator$ has finite hallucinations for language $K$ if $\abs{\generator(S)\setminus K}<\infty$
        \item[$\triangleright$] \textit{Infinite Hallucinations:} Generator $\generator$ has infinite hallucinations for $K$ if $\abs{\generator(S)\setminus K}=\infty$.
    \end{itemize}
    \cref{fig:characterization} summarizes our characterizations for different notions of breadth (along rows: exact breadth, approximate breadth, no breadth) and different amoungs of hallucinations (along columns: no hallucinations, finite hallucinations, and infinite hallucinations).
    \begin{figure}[!htb]
        \centering
        \vspace{-4mm}
        \includegraphics[width=0.8\linewidth,trim={6cm 10cm 2.5cm 4.5cm},clip]{figures/tab-v7.pdf}
        \caption{
            \textit{Characterizations of All Possible Notions of Generation:}\quad
            This figure lists all possible notions of language generation (at a certain granularity) and their characterizing conditions.
            Rows capture breadth (\ie{}, how many elements are missed from the target language).
            Columns capture the extent of hallucinations (\ie{}, how many elements outside of the target language are included).
            Generation becomes easier when moving down rows and/or right along columns.
            The notion of infinite coverage requires $|K \cap \supp(G)| = \infty$ (\Cref{def:infinite-coverage}).
            \vspace{-4mm}
        }
        \label{fig:characterization}
    \end{figure}
    \begin{proof}[Proof Sketch for Results in \cref{fig:characterization}]
        To achieve notions in the last column, one can generate the whole domain (\ie{}, ensure $\supp(\generator)=\cX$). 
        To achieve notions in the last row, one can use an extension of \cite{kleinberg2024language}'s algorithm from \cref{lem:upper-bound-infinite-coverage}. 
        It remains to explain the results in the top $2\times 2$ cells.
        Among these the two results in the first column are from \cref{Main}.
        For the remaining two results (the first two in the second column): 
            the lower bound follows from \cref{apx:proofs:lowerBounds:finitenonuniqueness} since both of these notions satisfy finite non-uniqueness (\cref{def:finiteNonUniqueness}).
            The upper bounds {are presented in \cref{apx:proofs:upperBounds}.}
    \end{proof}
    
    \subsection{Generation with Breadth, Stability, and Hallucinations}\label{apx:landscape:stability:hallucinations}
        Next, as in the previous section, to illustrate the generality of our techniques. For stable generators, we use them to give necessary and/or sufficient conditions for several new notions of generation with stable generators. 
        \cref{fig:characterization-stability} summarizes our results for different notions of breadth with stable generators (along rows: exact breadth, approximate breadth, no breadth) and different amoungs of hallucinations (along columns: no hallucinations, finite hallucinations, and infinite hallucinations).
        Interestingly, our results also show that if we allow for  finitely many hallucinations while missing no elements from the target language, stable generation is still \mbox{characterized by the weak Angluin's condition.} 
        
        Unlike the case of unstable generation, we do not have a complete characterization for every cell of \cref{fig:characterization-stability}.  
        It is an interesting direction to characterize all the remaining cells. %

\begin{figure}[!htb]
    \centering
    \vspace{-2mm}
    {\includegraphics[width=0.8\linewidth,trim={5cm 3cm 2cm 6.5cm},clip]{figures/stability_table.pdf}}
    \caption{
        {\textit{Stability Under All Possible Notions of Generation:}\quad
        This figure lists all possible notions of language generation (at a certain granularity).
        Rows capture the extent of breadth (\ie{}, how many elements are missed from the target language).
        Columns capture the extent of hallucinations (\ie{}, how many elements outside of the target language are included).
        Generation becomes easier as one moves down the rows and/or to the right along columns.
        For the yellow cell, we have shown that 
        not all countable collections admit a stable generator that satisfies this notion
        of breadth, but we do not have a condition that characterizes it.  For the gray cell, 
        we do not know whether all collections satisfy this notion, and we do not have a characterization.
        The notion of infinite coverage refers to a generator whose support satisfies $|K \cap \supp(G)| = \infty$ (see also \Cref{def:infinite-coverage}).}
    }
    \label{fig:characterization-stability}
\end{figure}
\begin{proof}[Proof Sketch for Results in \cref{fig:characterization-stability}]
    To achieve any notion in the last column, it is sufficient to generate the whole domain (\ie{}, ensure $\supp(\generator)=\cX$).
    Unlike the case of unstable generators, 
    achieving the notions in the last row is non-trivial. In particular, we show that there exists
    a collection for which no stable algorithm
    can achieve infinite coverage (\cref{thm:impossibility-stable-generation}).
    It remains to overview the results in the top left $2\times 2$ cells.
    Among these, the two results in the first column are from \cref{Main:stable}.
    For the remaining two results:
        the lower bound follows from \cref{apx:proofs:lowerBounds:finitenonuniqueness} since both notions satisfy finite non-uniqueness (\cref{def:finiteNonUniqueness})
        and the upper bound algorithms is as below:
        \begin{itemize}
            \item[] The algorithm that achieves these notions is straightforward adaptation of \cref{lem:upper-bound-approximate-breadth}\\ that does not {drop} the elements $S_t \cup \inbrace{x_1,\ldots,x_t}$ from the set it outputs.
        \end{itemize}
        \begin{center}
        
        \end{center}
\end{proof}

\section{Implication for the Statistical Setting}
    \label{apx:statistical}
    In this section, we discuss the implications of our results in the statistical setting.

    In this setting, there is a countable language collection $\cL$,
    a ``valid'' distribution $\cP$ supported on a language $K \in \cL,$ and the generating algorithm takes 
    as input string drawn \iid{} from $\cP.$
    For every different notion of breadth, one can define an error function for the generating
    algorithm $\inparen{\generator_n}_{n \in \N}$ as
    \begin{equation}\label{eq:error-function}
        \er\inparen{\generator_n} = \ind\inbrace{\lnot P(\generator_n)} \,,
    \end{equation}
    where $P(\cdot)$ is a predicate defined based on the underlying 
    notion of breadth and its value is $\mathrm{True}$ if the breadth
    property is achieved by $\generator_n$ and $\mathrm{False}$, otherwise.

    Given this definition \eqref{eq:error-function}, \citep{kalavasis2025limitslanguagegenerationtradeoffs} define the error rate for generation with breadth via the
    universal rates framework of \citet{bousquet2021theory}.

    \begin{definition}[Error Rate \citep{bousquet2021theory}]
        Let $\cL$ be a countable collection of languages, $\er$ be an error
        function defined in \cref{eq:error-function}, and $R\colon\N \rightarrow [0,1]$ be
        a rate function such that $\lim_{n \rightarrow \infty} R(n) = 0$. We say that rate $R(\cdot)$ is achievable
        for $\cL$ if there exists a generating algorithm
          $\generator = \inparen{\generator_n}$ such that
            \[
    \forall~ \cP \in \mathrm{Val}(\cL)~~ \exists~ C ,c  > 0 \quadtext{such that}
    \E\insquare{\er(\generator_n)} \leq C \cdot R(c \cdot n)\quad \forall n \in \N \,,
    \]
    where $\mathrm{Val}(\cL)$ the set of all valid distributions
    with respect to $\cL.$ 
    Conversely, we say that no rate faster than $R(\cdot)$ is achievable
    for $\cL$ if for any generating algorithm $\generator = \inparen{\generator_n}$ there exists a valid
    distribution $\cP$ and $c, C>0$ such that $\E\insquare{\er\inparen{\generator_n}} \geq C\cdot R(c\cdot n),$ for infinitely many $n \in \N.$
    We say that no rate is achievable for $\cL$
    if for any generating algorithm $\generator = \inparen{\generator_n}$ there exists
    a valid distribution $\cP$ such that $\limsup_{n \rightarrow \infty} \E\insquare{\er\inparen{\generator_n}} > 0.$ 
    \end{definition}
    \citep{kalavasis2025limitslanguagegenerationtradeoffs} proved
    bounds in this statistical setting for language identification, generation
    with exact breadth for algorithms for which the MOP is decidable,\footnote{Recall this is a mild technical condition that requires that the generating algorithm can answer queries about whether a string $x$ is in its support.}
    and generation with approximate breadth for algorithms that
    are stable in the limit,\footnote{Roughly speaking, stability means that after finitely many steps, the support of the distribution outputted by the generating algorithm does not change. For the formal definition, see \cref{def:stable-generators}.} and for which the MOP is decidable. To get these results, \citep{kalavasis2025limitslanguagegenerationtradeoffs} showed connections
    between the online setting considered in the previous sections and the statistical setting.
    Using the new results in this work, and the results of \citep{kalavasis2025limitslanguagegenerationtradeoffs}, we can get 
    characterizations for the statistical rates under these two notions of
    breadth removing the requirement for decidability of the MOP oracle
    and stability of the generating algorithm.

    \begin{theorem}[Rates for Generation with Exact Breadth]
        For any non-trivial collection of languages $\cL$ no rate faster than $e^{-n}$ is achievable for generation with exact breadth.
        Moreover, For any collection that is identifiable
        in the limit, there exists an algorithm
        that achieves generation with exact
        breadth at rate $e^{-n}.$ Conversely,
        for any non-identifiable collection, no
        rate is achievable for generation with exact breadth. 
    \end{theorem}
    For the non-triviality requirement, we
    refer the interested reader to \citep{kalavasis2025limitslanguagegenerationtradeoffs}. The $e^{-n}$ lower bound and upper bound follow immediately from {their results}. The lower bound for no rates achievable follows from the approach of \citep{kalavasis2025limitslanguagegenerationtradeoffs} {(with a few modifications 
    in their construction)} and {\cref{thm:uniqueness}}. 
   {We sketch the modifications here and omit a detailed proof.}
    \begin{itemize}
        \item \citep{kalavasis2025limitslanguagegenerationtradeoffs} make use
        of a construction of \citep{angluin1988identifying} which connects
        the adversarial setting ``in-the-limit'' to the statistical setting ``in-the-limit'' (Theorem 5.6 in their paper) for language identification. 
        A similar result can be shown for generation with exact breadth.

        \item \citep{kalavasis2025limitslanguagegenerationtradeoffs} make use of
        majority votes over learners that identify the target language. In Lemma 5.8 they use 
        the voting scheme, (a modification of) Angluin's result \citep{angluin1988identifying}, and the Borel-Cantelli lemma to show that no rate is achievable for language identification, for collections that do not satisfy Angluin's criterion (\Cref{def:angliun-criterion}). The same
        approach can be used to derive the lower bound for generation with
        exact breadth, by using a slightly different majority voting scheme.
        At a very high level,  following \citep{kalavasis2025limitslanguagegenerationtradeoffs}\footnote{The same approach has been used extensively in the universal rates literature, starting from \citep{bousquet2021theory}.} we
        split the dataset into different batches and train the generating algorithm, and we can show that for large
        enough $n,$ a $c$-fraction of these generators satisfies the generation with exact breadth property {(for, \eg{}, $c>\nfrac{2}{3}$)}. In order to combine their outputs,
        we define an (implicit) distribution as follows: we keep sampling from all the batches until a $c$-fraction of them outputs the same element. It is not hard to see that \textbf{(i)} this process terminates in finite time,\footnote{One small complication is that if a $c$-fraction does not satisfy the desired property, the algorithm might not terminate. To fix that, in every step we either terminate with probability $\nicefrac{1}{2}$ or we do the sampling strategy we described with probability $\nicefrac{1}{2}.$ If we terminate, we run the algorithm from \citep{kleinberg2024language} to generate a valid string from $K$.} \textbf{(ii)} only elements of $K$ have positive probability of being outputted, \textbf{(iii)}
        every element of $K$ has a positive probability of being outputted.
    \end{itemize}
    A similar result can be obtained for language generation with approximate
    breadth, using the criterion from \cref{def:suffCondition}.

       \begin{theorem}[Rates for Generation with Approximate Breadth]
        For any non-trivial collection of languages $\cL$ no rate faster than $e^{-n}$ is achievable for generation with approximate breadth.
        For any collection that satisfies \cref{def:suffCondition}, there exists an algorithm
        that achieves generation with approximate
        breadth at rate $e^{-n}.$ Conversely,
        for any collection that does not \cref{def:suffCondition}, no
        rate is achievable for generation with exact breadth. 
    \end{theorem}
    The above pair of results provides statistical rates for language generation with exact and approximate breadth. 
   {Obtaining statistical rates for unambiguous generation is an interesting direction.}

\section{Proofs of Lower Bounds}
    \label{apx:proofs:lowerBounds}
    \label{sec:proofs:lowerbounds}
    In this section, we prove our lower bound results.
    
    \subsection{Lower Bound with Uniqueness (Proof of \cref{thm:uniqueness})}
        \label{apx:proofs:lowerBounds:uniqueness}
        In this section, we prove \cref{thm:uniqueness}.
        Recall that this requires to prove that the following: if a collection $\cL$ violates Angluin's condition, then no generator can achieve a property $P$ satisfying uniqueness in the limit for $\cL$.

        \begin{figure}[!bth]
            \centering
            \vspace{-5mm}
            \includegraphics[width=0.475\linewidth,trim={0cm 0cm 7cm 0cm},clip]{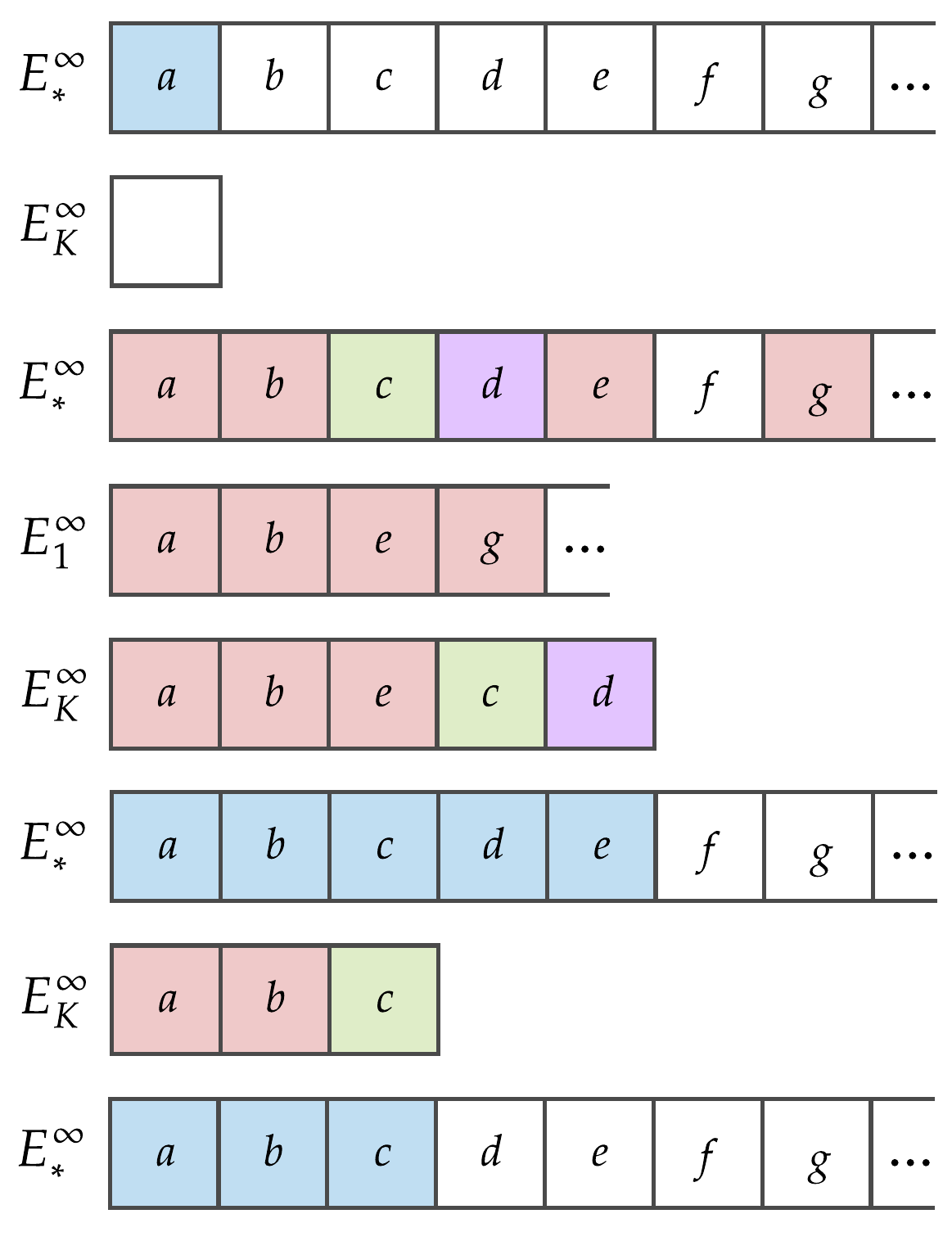}
            \caption{
                \textit{Illustration of the Construction in the Proof of \cref{thm:uniqueness}.}\quad 
                Fix any enumeration $a,b,c,d,e,f,g,\dots$ of the language $L^*$, depicted in the first row.
                The enumeration of $K$ is initially empty in the construction and this is depicted in the second row. 
                To begin the construction, we apply the contrapositive to Angluin's condition with $T = \{a\}$ (\ie{}, with the set highlighted in \textcolor{royalBlue}{blue} in the first row).
                This results in a language $L_1$ that contains $T$ and is a subset of $L^*$.
                For this illustration, suppose that the enumeration of $L_1$ is as presented in the fourth row. 
                The elements shared between $L_1$ and $L^*$ are highlighted in \textcolor{niceRed}{red} in the third row. 
                From the fourth row, we can see that the strings in $L_1$'s enumeration, $E_1^*$, follow the same relative order as in $E_*^\infty.$
                Further, note that $c,d,$ and $f$ are skipped from the enumeration since they do not belong to $L_{1}$ (\ie{}, they are not highlighted in \textcolor{niceRed}{red}). 
                Now, the algorithm in the proof is trained on the enumeration $E_{1}^\infty$ (Subphase A), and we consider two cases: 
            \textbf{Case (i):} Assume that after seeing element $e$, the algorithm {achieves property $P$}. Then we update $E_K^\infty$ by adding all elements of $E_1^\infty$ until $e$ and then add all the elements that we skipped from $E_*^\infty;$ this is shown in the fifth row where we added $c$ and $d$.
            This scenario corresponds to Subphase B.1 in the proof since at least one element from the enumeration of $E_*^\infty$ was skipped during Subphase A.
            Next, we again apply the contrapositive to Angluin's condition.
            This time, we set $T = \{a,b,e,c,d\}$ (denoted in \textcolor{royalBlue}{blue} in the sixth row), and, then repeat the process. 
            \textbf{Case (ii):} Assume that the algorithm {achieves property $P$} after seeing $b$. Then, we update $E_K^\infty$ by adding $a,b$ and then the first element that is not in $L_1$, \ie{}, $c$. 
            This is depicted in the seventh row.
             This scenario corresponds to Subphase B.2 in the proof since no strings from $E_*^\infty$ were skipped during Subphase A.
             Next, we again apply the contrapositive to Angluin's condition.
             This time, we use $T = \{a,b,c\}$ (denoted in \textcolor{royalBlue}{blue} in the last row) and repeat the process. 
            }
            
            \label{fig:constructionMainTheorem}
        \end{figure}

    \begin{proof}[Proof of \Cref{thm:uniqueness}]
        For any enumeration $E$, we use the notation
        $E(i)$ to denote its $i$-th element, $E(1:i)$ to denote its first $i$
        elements, and $E(i:\infty)$ to denote all but the first $i-1$ elements.
        Since $\cL$ is not identifiable in the limit, it does not satisfy Angluin's
        condition (\Cref{def:angliun-criterion}). 
        {Hence, there exists a language $L^*\in \cL$ such that the following holds:} 
        \[
            {\text{for all finite subsets $T\subseteq L^*$}\,,~~ 
            \text{there exists a language $L_T \in \cL$}\,,~~ 
            T \subseteq L_T
                ~~\text{and}~~
            L_T \subsetneq L^*\,.}
            \yesnum\label{eq:angluinViolation:gen}
        \]
        Fix $L^*\in \cL$ to be any language for which this holds.
        Let $E_{*}^\infty$ be an
        arbitrary enumeration of $L^*$, without repetitions.
        Let $K$ and $E_K^\infty$ respectively denote the target language and {its enumeration} that we will construct to show the impossibility result.

        We will show that for {any} generating algorithm $\generator=(\generator_n)$ 
        there exists a choice of the target language {$K$ in $\cL$ (which may be different from $L^*$)} and
        an enumeration of it {such that if $K$ is the target language and the adversary provides enumeration $E_K^\infty$ to $\generator$, then} the algorithm $\generator$ cannot generate with breadth in the limit.  

        {We will construct the enumeration iteratively and select $K$ based on the generating algorithm.
        The construction of the enumeration proceeds in multiple (possibly infinite) phases.
        }
        At any point
        $t \in \N$ of the interaction, we denote by $S_t$ the set of elements
        enumerated so far.

        \paragraph{Phase 1 of Construction.}
        To construct the first phase, we present the generator with the first element
        of the enumeration of $L^*$, \ie{}, $x_{i_1 } \coloneqq E_{*}^\infty(1)$. Let $L_{j_1}$ be some language such that $x_{i_1} \in L_{j_1}$ and $L_{j_1} \subsetneq L^*$, \ie{}, it is a proper subset of $L^*$.
        Notice that such a language is guaranteed
        to exist by picking $T = \{x_{i_1}\}$ in the violation of Angluin's condition \eqref{eq:angluinViolation:gen}. 
        \begin{itemize}
            \item \textbf{Subphase A (Enumerate $L_{j_1}$ Until Generator Generates with Breadth from $L_{j_1}$):}~~
                Consider an enumeration $E_1^\infty$ of the language $L_{j_1}$ that is constructed by traversing $E_*^\infty$ and using the elements of $L_{j_1}$ that appear in it, in the same order as they appear, \ie{}, for every $i \in \N$ it holds that $E_{1}^\infty(i)$ is the $i$-th element of $L_{j_1}$
                that appears in $E_*^\infty$.
                Notice that this is indeed a valid enumeration of $L_{j_1}$ {as $L_{j_1}$ is a subset of $L^*$}.
                At any round $t$ of the first phase, the adversary presents
                the element $E^\infty_1(t)$ to the generator.
                
                Consider two cases: i) either there is some finite $t_{1} \in \N$ such that 
                {$\generator_{t_1}$ {achieves property $P$ for} $L_{j_1}$}
                or ii) there is no
                such $t_1 \in \N.$ In the latter case, we pick the target language $K = L_{j_1}$ and the target enumeration $E_K^\infty = E_1^\infty$, and the lower bound follows {since we have found a pair of $K$ and $E_K^\infty$ for which the generator never achieves {property $P$}}. 
                Hence, assume that we are in the former case, and
                let $\hat x_1$ be the first element of $E_1^\infty$ for which
                the condition holds.
                {Note that, at this point, $\generator_{t_1}$ does \textit{not} {achieve property $P$ for} $L^*$ since {$P$} satisfies the uniqueness criterion and $L_{j_1}\neq L^*$.}
                Further, note that $S_{t_1}$ is the set of strings shown to the generating algorithm after which it starts to generate with breadth from $L_{j_1}$.
        \end{itemize}
        Let $\hat S_1$ be the set of elements of  $E^\infty_*$ that appear before $\hat x_1$ in $E^\infty_*$ and have not appeared in $S_{t_1}$. 
        If $\hat{S_1}\neq \emptyset$, we go to Subphase B.1 and, otherwise if $\hat{S_1}=\emptyset$, we go to Subphase B.2.
        \begin{itemize}
            \item \textbf{Subphase B.1 (Add Any Skipped Elements):}~~ 
        We will use the set $\hat S_1$
        to extend the construction of the target enumeration $E_K^\infty$.
        To do this, we enumerate
        the elements from $\hat S_1$ in an arbitrary order and we fix
        the prefix of the target enumeration $E_K^\infty$ to be $(S_{t_1}, \hat S_1).$ Notice that this step is well-defined since we are only adding to the already constructed enumeration.  Let $\hat t_1$ be the total number of elements enumerated so far. Notice
        that $\hat t_1 = \infty$ if and only if Case i) {(from Subphase A)} holds, in which
        case the lower bound already follows. Hence, assume
        for the continuation of the proof that $\hat t_1 < \infty.$
        {Now we terminate the first phase (without going to Subphase B.2).}
        \item \textbf{Subphase B.2 (If Nothing Skipped Enumerate An Element Outside $L_{j_1}$):}~~ Notice that $\hat S_1 = \emptyset$ if and only if we did not skip any element of $E_*^\infty$ during the traversal {in Subphase A}.
        If we indeed did not skip elements of $E_*^\infty$ we continue 
        traversing it and adding elements to $E_K^\infty$ in the same
        order as we see them in $E_*^\infty$
        until we find some element that does not belong to $L_{j_1}.$ We also include this element in the enumeration $E_K^\infty$,
        we fix $\hat t_1$ to be the number of elements enumerated so
        far and we terminate the first phase.
        \end{itemize}
        Notice that so far in our construction, we have enumerated the first $\hat t_1$ elements 
        of $E_*^\infty.$

        Now we continue our construction inductively for phases $\ell=2,3,\dots$.
        Consider any $\ell\geq 2$.
        Suppose our construction continued from Phase 1 until Phase $\ell$.
        Then, Phase $\ell+1$ of our construction is as follows.
        
        \paragraph{Phase $\ell+1$ of Construction.}
        For the $(\ell+1)$-th phase, consider the set $E_*^\infty(1:\hat t_{\ell})$ that
        has been enumerated so far. 
        By construction, 
        \[
            E_*^\infty(1:\hat t_\ell) \not\subseteq L_{j_\ell}\,,\quad 
            E_*^\infty(1:\hat t_\ell) \subseteq L^*\,,
            \quadand
            E_*^\infty(1:\hat t_\ell) \text{ is finite}\,.
        \]
        We will now apply the violation of Angluin's condition \eqref{eq:angluinViolation:gen} with $T =  E_*^\infty(1:\hat t_\ell).$
        This means that
        there must exist some {$j_{\ell+1} \not\in \inbrace{j_1,j_2,\dots,j_\ell}$} such
        that 
        \[
            L_{j_{\ell+1}}\in \cL\,,\quad 
            L_{j_{\ell+1}} \subsetneq L^*\,,\quadand
            E_*^\infty(1:\hat t_\ell) \subseteq L_{j_{\ell+1}}\,.
        \]
        We now perform analogs of each subphase in Phase 1.
        \begin{itemize}
            \item \textbf{Subphase A (Enumerate $L_{j_{\ell+1}}$ Until Generator Generates with Breadth from $L_{j_{\ell+1}}$):}~~
                Consider an enumeration
                $E_{\ell+1}^\infty$ of $L_{j_{\ell+1}}$ whose first $\hat t_{\ell}$ strings are $E_*^\infty(1:\hat t_{\ell})$ and whose remaining strings are constructed by traversing
                $E_*^\infty(\hat t_{\ell}+1:\infty)$ and selecting strings that belong to $L_{j_{\ell+1}},$ in the same order as they appear in $E_*^\infty$. 
                Notice
                that this is indeed a valid enumeration of $L_{j_{\ell+1}}$ {as $L_{j_{\ell+1}}$ is a subset of $L^*$}.
                {At any round $t$ of this phase, the adversary presents the element $E_{\ell+1}^\infty(t+\hat{t}_\ell)$ to the generator. }
                
                Consider two cases: i) either there is some finite $t_{\ell+1} \geq  \hat t_\ell + 1$ such that 
                {$\generator_{t_{\ell+1}}$ {achieves property $P$ for} $L_{j_{\ell+1}}$}
                or ii) there is no
                such $t_{\ell+1} \in \N.$ In the latter case, we pick the target language $K = L_{j_{\ell+1}}$ and the enumeration $E_K^\infty = E_{\ell+1}^\infty$, and the lower bound follows {since we have found a pair of $K$ and $E_K^\infty$ for which the generator never achieves {property $P$}}. 
                Hence, assume that we are in the former case, and let $\hat x_{\ell+1}$ be the first element of $E_{\ell+1}^\infty$ for which
                the condition holds. 
                {Note that, at this point, $\generator_{t_{\ell+1}}$ does \textit{not} {achieve property $P$ for} $L^*$ since {$P$} satisfies the uniqueness criterion and $L_{j_{\ell+1}}\neq L^*$.}\footnote{{Note that, at this point, the generator may indeed be generating with breadth from one of the earlier languages $\sinbrace{L_{j_{1}},\dots, L_{j_{\ell-1}}, L_{j_{\ell}}}$. The proof only requires that the generator does not generate with breadth from $L^*$.}}
                Further, note that $S_{t_{\ell+1}}$ is the set of strings shown to the generating algorithm after which it starts to generate with breadth from $L_{j_{\ell+1}}$.
        \end{itemize}
        Let $\hat S_{\ell+1}$ be the set of strings of 
        $E^\infty_*$ that appear before $\hat x_{\ell+1}$ in $E^\infty_*$ and have not appeared
        in the enumeration $S_{t_{\ell+1}}$. 
        If $\wh{S}_{\ell+1}\neq \emptyset$, we go to Subphase B.1 and, otherwise if $\wh{S}_{\ell+1}=\emptyset$, we go to Subphase B.2.
        \begin{itemize}
            \item \textbf{Subphase B.1 (Add Any Skipped Elements):}~~
                {We will use $\hat{S}_{\ell+1}$ to extend the construction of the target enumeration $E_K^\infty.$}
                To do this, we enumerate the elements from $\hat S_{\ell+1}$ in an arbitrary order and we fix the prefix of the target enumeration $E_K^\infty$ to be $(S_{t_{\ell+1}}, \hat S_{\ell+1}).$ 
                Notice that this step is well-defined since we are only adding to the already constructed enumeration.
                Let $\hat t_{\ell+1}$ be the set of elements enumerated so far. Notice
        that $\hat t_{\ell+1} = \infty$ if and only if Case i) (from Subphase A) holds, in which
        case the lower bound already follows. 
        Hence, assume
        for the continuation of the proof that $\hat t_{\ell+1} < \infty.$
        {Now we terminate the $(\ell+1)$-th phase without going to Subphase B.2.}
                
            \item \textbf{Subphase B.2 (If Nothing Skipped Enumerate An Element Outside $L_{j_{\ell+1}}$):}~~
            Notice that $\hat S_{\ell+1} = \emptyset$ if and only if we did not skip any element of $E_*^\infty$ during the traversal {in Subphase A}.
            If we indeed did not skip elements of $E_*^\infty$ we continue 
            traversing it and adding elements to $E_K^\infty$ in the same order as we see them in $E_*^\infty$
            until we find some element that does not belong to $L_{j_{\ell+1}}.$
            We also include this element in the enumeration $E_K^\infty$, we set $\hat t_{\ell+1}$ to be the number of elements enumerated so far and we terminate Phase $\ell+1$.
        \end{itemize}
    Notice that so far we have enumerated the first $\hat t_{\ell+1} > \hat t_\ell + 1$ elements  of $E_*^\infty.$
    
    \paragraph{Inductive Argument.}
    As explained, we continue the construction of the target enumeration inductively. 
    If there is some phase $\ell$ such that Case ii) (in Subphase A) is activated, then the lower bound follows.
    Let us now assume that Case ii) is not activated for any phase $\ell \in \N$. 
    Then, we have constructed an enumeration of $L^*$ (by construction of the sets $S_{t_\ell}$ and $\hat S_{\ell}$ for each $\ell\in \N$) such that
    {$\generator_t$ does not {achieve property $P$ for} $L^*$} for infinitely many $t \in \N.$ 
    {Now, the lower bound follows by setting} the target language $K = L^*$ and the target enumeration to the one we have constructed inductively over all phases.
    \end{proof}
    \subsection{Lower Bound with Finite Non-Uniqueness  (Proof of \cref{thm:finiteNonUniqueness})}
        \label{sec:FiniteNonUniqueness}
        \label{apx:proofs:lowerBounds:finitenonuniqueness}
    In this section, we prove \cref{thm:finiteNonUniqueness}. 
    \begin{proof}[Proof of \cref{thm:finiteNonUniqueness}]
        The proof of this lower bound uses the construction in the proof of \cref{thm:uniqueness} with one change: now the language $L_T$ (introduced at the start of the proof) is the language determined by the contrapositive to the weak Angluin's criterion (\cref{def:weakAngluin}) and not the contrapositive to the (usual) Angluin's criterion (\cref{def:angliun-criterion}).
        Concretely, the contrapositive to the weak Angluin's criterion implies that there exists a language $L^*\in \cL$ such that the following holds: %
        \[
            \forall T\subseteq L^*\,,\quad 
            \exists L_T \in \cL\,,\quad \text{such that}\quad 
            T \subseteq L_T\,,\quad 
            L_T \subsetneq L^*\,,\quadand
            \abs{L^*\setminus L_T} =\infty\,.
            \yesnum\label{eq:angluinViolation2}
        \]
        We will use this language $L^*$ and proceed with the construction without change.
    
        Having completed the construction, we proceed to the proof.
        The only place in which the proof uses a property of the criterion for breadth is when it invokes the uniqueness criterion with respect to the pair of languages $L_T$ and $L^*$ (once in Subphase A of each phase).
        Here, $T$ is the set $E_*^\infty(1)$ in the first phase and $E_*^\infty(1:\hat t_{\ell})$ in the $\ell$-th phase.
        Now, we cannot directly invoke the uniqueness criterion {since $P$ does not satisfy it.} %
        However, since $\abs{L^*\setminus L_T} =\infty$ and since {property $P$} satisfies the finite non-uniqueness criterion, we can conclude that no generator can {achieve property $P$ for} both $L^*$ and $L_T$ simultaneously, as desired.
        Hence, we can use the finite non-uniqueness criterion in analyzing each phase of the construction and the result follows as in the proof of \cref{thm:uniqueness}.
    \end{proof}

    \subsection{Lower Bound for Approximate Breadth with Stability {(Proof of Lower Bound in \cref{thm:stable})}}
        \label{apx:proofs:lowerBounds:stability}
        In this section, we prove the lower bound in \cref{thm:stable}:     
        we show that if a collection $\cL$ violates Angluin's condition, then no generator can generate with approximate breadth from $\cL$.
        (Note that this as a corollary implies that no generator can generate with exact breadth.)

    \begin{proof}[Proof of lower bound in \cref{thm:stable}]
       
        {We will use the following corollary of the construction in the previous section.
        \begin{restatable}[]{corollary}{lemStructural}
        \label{lem:structural}
            Let $\cL$ be a countable collection of languages that is not identifiable in the limit.
            Let $\generator=(\generator_n)$ be a generating algorithm.
            If $\generator$ generates with approximate breadth from $\cL$ in the limit, then there is a language $L^*\in \cL$, an enumeration $E^*$ of $L^*$, a sequence of distinct languages $L_{\ell_1}, L_{\ell_2}, \dots\in \cL$, and a \emph{strictly} increasing sequence $t(1), t(2), \dots\in \N$,  such that the following holds. 
            \begin{itemize}
                \item For each $i\in \N$, $L_{\ell_i}$ is a proper subset of $L^*$, \ie{}, $L_{\ell_i}\subsetneq L^*$; and 
                \item Given strings from $E^*$ as input, for each $i\in \N$, $\generator_{t(i)}$ generates with approximate breadth from $L_{\ell_i}$.
            \end{itemize}
        \end{restatable} 
        }
        
        \noindent Consider the construction in {in the above corollary.}
        Let $K=L^*$ and suppose that the adversary follows the enumeration $E^*$.
        
        Let $C_B,C_S\colon \N\to \N$ be two counters: for each $t$, $C_B(t)$ counts the number of values $1\leq i\leq t$ for which $\generator_i$ does \textit{not} generate with approximate breadth from $L^*$ and $C_S(t)$ counts the number of values $2\leq i\leq t$ for which $\supp(\generator_i)\neq \supp(\generator_{i-1})$.
        In other words, $C_B(t)$ is the number of times $\generator$ does not generate with approximate breadth from $L^*$ in the first $t$-steps and $C_S(t)$ is the number of times $\generator$ changes its support in the first $t$-steps.

        Toward a contradiction suppose that $\generator$ is stable and generates {with approximate breadth} from $K$ in the limit (when given the enumeration $E^*$). 
        This, by definition, implies that 
        \[
            \lim_{t\to\infty} C_B(t)<\infty 
            \qquadand
            \lim_{t\to\infty} C_S(t)<\infty\,.
            \yesnum\label{eq:approxbreadth_stability_counter}
        \]
        The former implies that there are only finitely many values of $i\in \N$ such that $\generator_{t(i)}$ does not generate with approximate breadth from $L^*$ 
        (where $t(i)$ are from \cref{lem:structural}).
        Thus, after discarding a sufficiently large
        finite prefix of $t(i), i \in \N,$ and re-indexing
        we see that there are infinitely many values, say, $\tau(1)<\tau(2)<\dots\in \N$, such that, for each $i$, $\generator_{\tau(i)}$ generates with approximate breadth from $L^*$ and $L_{\ell_i}.$
        Since $\generator_{\tau(i)}$ generates with approximate breadth from both $L^*$ and $L_{\ell_i}$ and $L_{\ell_i}\subsetneq L^*$, it follows that: for each $i\in \N$,
        \[
            L_{\ell_i} = \supp(\generator_{\tau(i)}) \cup R 
            \qquadwhere
            R\subseteq L^* \setminus \supp(\generator_{\tau(i)})\,.
            \yesnum\label{eq:structure_of_Li}
        \]
        Fix any $i$. 
        Let 
        \[
            s(i)\coloneqq \abs{L^* \setminus \supp(\generator_{\tau(i)})}\,.
        \]
        Since $\generator_{\tau(i)}$ generates with approximate breadth from $L^*$, $s(i)<\infty$.
        We claim that 
        \[
            \supp(\generator_{\tau(i)})
            \neq 
            \supp(\generator_{\tau(i + j)})
            \qquadtext{for some}
            1 \leq j \leq S(i) \coloneqq2^{s(i)}+1\,.
            \yesnum\label{eq:claim_support_change}
        \]
        \paragraph{Proof of \cref{eq:claim_support_change}.}
        To see this, toward a contradiction, suppose that 
        \[
            \supp(\generator_{\tau(i)})=\supp(\generator_{\tau(i+1)})=\dots=\supp(\generator_{\tau_{i+S(i)}})\,.
        \]
        This combined with \cref{eq:structure_of_Li} implies that, for each $1\leq j\leq S(i)$, $L_{\ell_{(i+j)}}=\supp(\generator_{\tau(i)})\cup R_j$ for some finite set $R_j\subseteq L^*\setminus \supp(\generator_{\tau(i)})$.
        Since all of $L_{\ell_1},L_{\ell_2},\dots$ are different, it must hold that all of $R_1,R_2,\dots,R_{S(i)}$ are different.
        This is a contradiction since each $R_i$ is a subset of $R_i\subseteq L^*\setminus \supp(\generator_{\tau(i)})$ and there are only $S(i)-1=2^{s(i)}$ such subsets.

        \paragraph{Completing the Proof.}
            \cref{eq:claim_support_change} shows that, for each $i \in \N$, starting from the $\tau(i)$-th step, the support of the generator changes after finitely many steps.
            Since $\tau_1,\tau_2,\dots,\in\N$ is a strictly increasing and infinite sequence, this implies that the support of the generator changes infinitely often as it is provided more and more examples and, hence, $\lim_{t\to\infty} C_S(t)=\infty$ which contradicts the fact that $\generator$ is stable \eqref{eq:approxbreadth_stability_counter}.
            Hence, our assumption that $\generator$ is stable and generates with approximate breadth from $\cL$ in the limit must be false.
            Therefore, no stable generator can generate with approximate breadth from any non-identifiable collection.
    \end{proof}

\section{Proofs of Upper Bounds}
    \label{apx:proofs:upperBounds}
    \label{sec:proofs:upperbounds}
    In this section, we present new algorithms for generation required in our results (\cref{thm:approximate-breadth,Main,Main:stable,fig:characterization,fig:characterization-stability}).
    
    \subsection{Functional Upper Bound for Generation with Approximate Breadth}
        \label{sec:proofs:upperbounds:function}
        In this section, we present a function\footnote{Using the terminology of \citet{kleinberg2024language}, we refer to algorithms that have access to certain oracles (beyond membership oracle) specific to the collection $\cL$ as functions; reserving the term algorithm for algorithms which only require membership access to languages in $\cL$ (\ie{}, answer to questions of the form ``is $s\in L_i$?'').} that generates with approximate breadth from any countable collection $\cL$ satisfying weak Angluin's condition.
        This establishes the upper bound in \cref{thm:approximate-breadth}.
        
\begin{lemma}
[Function for Generation with Approximate Breadth]
\label{lem:upper-bound-approximate-breadth-function}
    Let $\cL$ be a countable collection of
    languages that satisfies \Cref{def:suffCondition}.
    Then, there exists a generating algorithm that, given access to a membership oracle for 
    $\cL$ and a subset oracle for $\cL$ (that given indices $i,j$ outputs Yes if $L_i\subseteq L_j$ and No otherwise), generates
    from $\cL$ with approximate breadth in the limit.
\end{lemma}
{This proof is inspired by the proof of Theorem B.2 in \citep{kalavasis2025limitslanguagegenerationtradeoffs}, the difference is that, instead of using Angluin's condition (\cref{def:angliun-criterion}), we use its weakening (\cref{def:weakAngluin}).}
\begin{proof}[Proof of \Cref{lem:upper-bound-approximate-breadth-function}]
        The algorithm $\algo{A}$ is illustrated below. 
        This algorithm follows the steps of the generation algorithm of \citep{kleinberg2024language} (see Steps 1 to 5). The only change is in its last Step 6 where it generates a random sample from the set of interest.

        \begin{mdframed}
            
            \noindent \textbf{for} $t\in \inbrace{1,2,\dots}$ \textbf{do:}
            \begin{enumerate}
                \item Observe element $x_t$ and let $S_t$ be the set of all elements observed so far.
                
                \item Construct a \textit{version space} $V_t$ consisting of all languages in $\cL_{\leq t}$ consistent with $S_t$, \ie{}, 
                \[
                    V_t\coloneqq \inbrace{L_j\colon 1\leq j\leq t\,,~~L_j\supseteq S_t}\,.
                \]
                \item[] \textit{$\#$ Define a language $L_i\in V_t$ to be \emph{critical} if $L_i$ is the smallest-index language in $V_t$ or $L_i$ is a subset of all languages preceding it in $V_t$, \ie{},  $L_i\subseteq L_j$ for all $1\leq j < i$.}

                \item If $V_t = \emptyset,$ \textbf{output} an arbitrary element of $\cX$ and \textbf{go} to the next iteration.
                
                \item Construct the set $C_t\subseteq V_t$ of all critical languages.
                \item[] \emph{$\#$ To construct the set of critical languages $C_t$ the algorithm needs access to the subset oracle.}
                 
                \item Let $L_i$ be the largest-indexed language in the set of critical languages $C_t$.
                
                \item \textbf{output}
                a sample from any distribution whose support is $L_i\setminus S_t$.
                This can be done in a computable fashion by first sampling a natural number $n$ from (\eg{}, the geometric distribution on $\N$) and then outputting the $n$-th string from $L_i\setminus S_t$.
            \end{enumerate}
        \end{mdframed} 
        
        \noindent Let $z$ be the first index such that $K=L_z$.
        The proposed algorithm generates with approximate breadth from $K$ when after some finite time $t^*,$ and for $t > t^*,$ the last language in the set of critical languages $C_t$, $L_i$  $= L_i(t)$, satisfies that 
        \[
            L_i\subseteq K \quadand
            \abs{K\setminus L_i} < \infty\,.
        \]
        This condition is implied by the following two conditions.
        \begin{enumerate}
            \item[(A)] $K$ is eventually included in set of critical languages $C_t$ and is never removed after that.
            \item[(B)] Eventually all the languages $L_j$ with $j > z$ that are in $C_t$ satisfy $L_j \subseteq K$ and $\abs{K\setminus L_j}<\infty$.
        \end{enumerate}
       
        \noindent Result (4.3) of \citep{kleinberg2024language} shows that there is a finite time $t_A$ after which Condition (A) holds.
        We will show that there is also a finite time $t_B$ after which Condition (B) holds.
        This shows that, for any $t\geq \max\inbrace{t_A,t_B}$, $\algo{A}$ generates with approximate breadth from $K$.

        \paragraph{Condition (B) holds after a finite time.}
            Since $\cL$  satisfies the weakening of Angluin's condition (\Cref{def:weakAngluin}), $K=L_z$ has a finite tell-tale set $T_z$, such that, any language $L\in \cL$ containing the tell-take $T_z$ satisfies one of the following:
            \begin{itemize}
                \item Either $L$ is not a proper subset of $K$;
                \item Or $L$ is a proper subset of $K$ and satisfies $\abs{K\setminus L} < \infty$.
            \end{itemize}
            (Recall that $T_z$ is not known to us; our proof will not need this.)
            Fix any $j>z$ and any time $t_B \geq t_A$ after which $K$ is guaranteed to be a critical language and after which $S_t\supseteq T_z$ (which happens at a finite time since $T_z$ is finite and, so, all elements of $T_z$ appear in the enumeration of $K$ at some finite time).
            Our goal is to show that for any $t\geq t_B$, and any $j>z$ for which $L_j$ is in $C_t$, it holds that 
            \[
                L_j \subseteq K
                \quadand
                \abs{K\setminus L_j}<\infty\,.
            \]
            By the definition of critical languages and the fact that $L_j$ appears after $K=L_z$ in the set of critical languages (as $j > z$), it follows that $L_j \subseteq K.$
            Hence, it remains to show that $\abs{K\setminus L_j}<\infty$.
            To see this, observe that since $L_j\in C_t$ and $C_t\subseteq V_t$, $L_j$ is in the version space $V_t$ and, hence, by the definition of $V_t$, $L_j\supseteq S_t.$
            Therefore, in particular, $L_j\supseteq T_z$ (as $S_t\supseteq T_z$).
            Now, \cref{def:weakAngluin} combined with the observation that $L_j\subseteq K$   implies that $\abs{K\setminus L_j}<\infty$ as required.
    \end{proof}

    \noindent Building on the result of \cite{kalavasis2025limitslanguagegenerationtradeoffs} (Corollary B.2 in their paper), the previous result shows that the function\footnote{To be precise, the function is that of \citep{kleinberg2024language} together with a process to sample from a language given membership access to it; see \eg{}, Step 6 in the Algorithm of \Cref{lem:upper-bound-approximate-breadth-function}.} of \cite{kleinberg2024language}
    with access to a subset query oracle achieves the
    ``best-of-three'' worlds for generation, without
    requiring any prior information about $\cL,$
    only subset and membership oracle access.

    \begin{corollary}
        Let $\cL$ be a countable collection of languages. 
        Exactly one of the following holds for the subset-oracle-based function of \cite{kleinberg2024language}.
        \begin{itemize}
            \item If $\cL$ satisfies Angluin's condition,
            the function generates with exact breadth in the limit. 

            \item If $\cL$ does not satisfy Angluin's condition but satisfies the weak Angluin's condition, the function generates with approximate breadth in the limit.

            \item If $\cL$ does not satisfy the weak
            Angluin's condition, the function
            generates with infinite coverage in the limit.
        \end{itemize}
    \end{corollary} 
    
    \subsection{Algorithmic Upper Bound for Generation with Approximate Breadth }
        \label{sec:proofs:upperbounds:membership}
         Next, we give an algorithm that generates with approximate breadth without requiring access to a subset oracle.
         This establishes an alternate proof of the upper bound in \cref{thm:approximate-breadth}.
\begin{lemma}
[Algorithm for Generation with Approximate Breadth]
\label{lem:upper-bound-approximate-breadth}
    Let $\cL$ be a countable collection of
    languages that satisfies \Cref{def:suffCondition}.
    Then, there exists a generating algorithm that, given access to a membership oracle for 
    $\cL$ and the tell-tale oracle from \cref{def:suffCondition}, generates
    from $\cL$ with approximate breadth in the limit.
\end{lemma}
\begin{proof}[Proof of \Cref{lem:upper-bound-approximate-breadth}]
    Let $S_n$ be the set of elements the adversary
    has enumerated up to round $n \in N.$
    For every $i, n \in \N$, let $T^i_n$ be the
    first $n$ elements enumerated from the tell-tale
    oracle when called on language $L_i.$ 
    Let also $x_1,x_2,\ldots,$ be an enumeration 
    of the domain $\cX.$  Our proof is reminiscent of Angluin's approach \citep{angluin1980inductive}, and the
    generating algorithm requires only one extra step,
    namely removing the elements $x_1,\ldots,x_n$
    from the support of the outputted distribution.
    However, due to the relaxed condition we are using, our analysis is more technically involved.
    
    For every round $n \in \N$, the generating algorithm
    constructs the sets $T^i_n$ using the tell-tale oracle 
    for all languages $L_i$ with $1 \leq i \leq n.$
    Let $g_n \in \N, 1\leq g_n \leq n,$ be the smallest number
    (if any) such that $S_n \subseteq L_{g_n}$ and $T_n^{g_n} \subseteq S_n.$ 
    If no such number exists, let $\generator_n$ be
    some arbitrary distribution. Otherwise,
    let $\generator_n$ be a distribution with 
    $\supp(\generator_n) = L_{g_n} \setminus \inparen{S_n \cup \inbrace{x_1,\ldots,x_n}}$.\footnote{{One can sample from this distribution in a computable fashion.}}

    \begin{mdframed}

            Fix a canonical enumeration $x_1, x_2,\dots$ of $\cX$.
            
            \vspace{2mm}

            \noindent \textbf{for} $n \in \inbrace{1,2,\dots}$ \textbf{do:}
            \begin{enumerate}
                \item Let $S_n$ be the set of all elements observed so far.
                
                \item Create the list $\cL_{\leq n} = \{L_1,\dots,L_n\}$.

                \item For each language $L_i$ in $\cL_{\leq n}, $ let $T^i = \mathsf{TellTaleOracle}(L_i), ~i \in [n]$.

                \item Truncate the outputs of the oracle and keep only their first $n$ elements
                \[
                T_n^i = (T^i(1),\dots,T^i(n)), ~ i \in [n]\,.
                \]
                
                \item Find smallest index $g_n \in \{1,\dots,n\}$ such that $S_n \subseteq L_{g_n}$
                and
                $T_n^{g_n} \subseteq S_n$.
                
                \item[] $\#$ \emph{This is the minimum indexed language in $\cL_{\leq n}$ that is consistent and its truncated tell-tale is contained in the observed elements.} 

                \item If no such $g_n$ exists, \textbf{output} an arbitrary point from $\cX$ and \textbf{go} to the next iteration.

                \item Otherwise, define a distribution $\generator_n$ with  $\supp(\generator_n) = L_{g_n} \setminus (S_n \cup \{x_1,\dots,x_n\})$.
                 \item[] $\#$ \emph{The intuition for removing the first $n$ elements $x_1,\dots,x_n$ of the canonical enumeration of $\cX$ is as follows. A bad scenario for our algorithm is that there exists some language $L_{g_n}$ in the enumeration of $\cL$ before $L_z = K$ such Step 5 will be stuck on $L_{g_n}$. Then we can guarantee that $|L_{g_n} \setminus K| < \infty$. Since this set is finite, by removing  parts of the enumeration of $\cX$ of increasing but finite size, we will eventually remove $|L_{g_n} \setminus K|$, and obtain a sampler that (i) is consistent and (ii) misses only finitely many elements from $K$.}
                 
                \item \textbf{Output} a sample from the distribution $\generator_n$.
            \end{enumerate}
        \end{mdframed} 
    
    \smallskip 
    
    \noindent We will show that this algorithm generates 
    with approximate breadth in the limit.
    Let $K$ be the target language and $z \in \N$ be the
    smallest number such that $L_z = K.$ 
    We consider two cases.
    
    \paragraph{Case A ($z=1$):}
    $S_n \subseteq L_1, \forall n \in \N$ 
    and since the tell-tale set $T^1$ of $L_1$ is finite and 
    the adversary presents a complete presentation of
    $K,$
    it holds that $T_n^1 \subseteq S_n$ for sufficiently
    large $n.$ 
    Thus, in the limit, it holds that $g_n=1$, thus $\supp(\generator_n) = L_1 \setminus \inparen{{S_n} \cup \{x_1,\ldots,x_n\}}$,
    and the proof is concluded by noting that $\supp(\generator_n) \subseteq K$ and $\abs{{S_n} \cup \{x_1,\ldots,x_n\}} < \infty,$
    for all sufficiently large $n.$

    \paragraph{Case B ($z>1$):}
    We now move on to the case $z > 1.$ 
    Then, for every language $L_i, 1 \leq i \leq z-1,$ that
    precedes $L_z,$ exactly one of the following holds:
    \begin{itemize}
        \item[(i)] either there exists some $x_{j_i} \in L_z$ but $x_{j_i} \notin L_i$, or 
        \item[(ii)]$L_z \subsetneq L_i.$
    \end{itemize}
    If Case \textbf{(i)} holds, then 
    there exists some $n_i \in \N$ such that $S_{n_i} ~{\not\subseteq}~ L_i.$ Thus, since there are finitely many languages
    before $z$ for which Case \textbf{(i)} holds, after
    finitely many $n \in \N$ all of them will have been
    contradicted by $S_n.$ Thus, we consider
    some $n_0 \in \N$ large enough so that for all $n \geq n_0$ every language
    $L_i, 1\leq i \leq z-1,$ for which $S_n \subseteq L_i$
    satisfies $L_z \subsetneq L_i.$
    
    Let $\cI = \inbrace{i_1,\ldots,i_\ell}$ be the set
    of the indices for which the previous holds.
    For every $j \in \cI,$ and for all $j' \in \N$
    for which the tell-tale set of $L_j$
    is a subset of $L_{j'},$ \ie{},
    $T^j \subseteq L_{j'},$
    one of the following 
    two cases hold by the definition of the weak Angluin's condition: (a) either $L_{j'}$ is not a proper subset of $L_j$ or (b) $\abs{L_{j} \setminus L_{j'}} < \infty.$
   
    Consider $j' = z$ and any $j \in \cI.$ Since,
    by construction, $L_z \subsetneq L_j,$
    the previous
    argument shows that either \textbf{(I)} $T^j \not\subseteq L_z$ or \textbf{(II)} $\abs{L_j \setminus L_z} < \infty.$
    
    If $j$ falls into
    Case \textbf{(I)} then for large
    enough $n$ it holds that $T_{n}^j \not\subseteq L_z$,  thus $T_{n}^j \not\subseteq S_n,$ and due to the way we have defined $g_n,$ $g_n \neq j.$\footnote{{Observe that if we had assumed the stronger \Cref{def:angliun-criterion} (Angluin's condition), then this step implies that we can identify $L_z$ in the limit, since only Case (I) is valid. This is exactly how the tell-tale-based algorithm of \citep{angluin1980inductive} works.}}
    Thus, we let $\cI'$ be the set of indices $j \in \N, 1 \leq j \leq z-1,$ such that $T^j \subseteq L_z$ and $L_z \subsetneq L_j$ and, hence, since we fall into Case \textbf{(II)} the previous argument implies that $\abs{L_{j}\setminus L_z} < \infty$ for each $j \in \cI'$. 
     
    We consider again two cases: if $\cI' = \emptyset,$ then for large enough $n$ it holds that $g_n = z.$ Hence, the correctness follows
     from the previous arguments. 
     
     We now handle the more complicated case $\cI' \neq \emptyset.$ 
     Let $j^*$ be the first element of $\cI'$.
    For large enough $n,$ the choice of $g_n$
    will stabilize to  $j^*$.
    To see this, notice that
    $S_n \subseteq L_{j^*}$ for all $n \in \N,$
    $T_n^{j^*} = T^{j^*}$ for sufficiently large $n$
    (since $T^{j^*}$ is finite), and since
    $T^{j^*} \subseteq L_{z}$ (and the adversary
    presents a complete presentation of $L_z$), for 
    large enough $n$ it holds that $T^{j^*}_n \subseteq S_n.$
    Thus, indeed for all sufficiently large $n$ it holds that $g_n = j^*.$ By definition of $\cI',$
    it holds that $\abs{L_{j^*} \setminus L_z} < \infty.$
    Let $x_{\ell_{j^*}}$ be the largest element of the enumeration
    of $\cX$ for which $x_{\ell_{j^*}} \in L_{j^*}$ but
    $x_{\ell_{j^*}} \notin L_z$ {(this always exists as $j^*\in \cI'$ and, hence, {$L_{z}\subsetneq L_{j^*}$ and $\abs{L_{j^*}\setminus L_z} < \infty.$})}. For $n \geq \ell_{j^*}$
    it holds that $L_{j^*} \setminus \inbrace{x_1,\ldots,x_n} \subseteq L_z$. This shows that, indeed, $\supp(\generator_n) \subseteq K$, for 
    large enough $n$, since we set $\supp(\generator_n) = L_{j^*} \setminus (S_n \cup \{x_1,\dots,x_n\})$.
    Moreover, since $L_z \subsetneq L_{j^*}$, 
    and $\abs{\inbrace{x_1,\ldots,x_n}} < \infty$, it holds that
    $\abs{L_z \setminus \inparen{L_{j^*} \setminus \inbrace{x_1,\ldots,x_n} }} < \infty,$ 
    for all $n \in \N$. Hence, the generator
    generates with approximate breadth from $K$ in the limit.
\end{proof}

\begin{remark}
    The generating algorithm that achieves approximate breadth in the limit for languages that satisfy the weak version of Angluin's condition has the property that the Membership Oracle Problem is decidable. Hence, by the results of \citep{kalavasis2025limitslanguagegenerationtradeoffs}, it cannot be stable, and, indeed, it is not since its support changes at each iteration.
\end{remark}
    
    \subsection{Extensions to Other Notions of Breadth}
        \label{sec:proofs:upperbounds:exhaustive}
        In this section, we generalize the results from the previous two sections to give algorithms that achieve exhaustive generation for countable collections satisfying weak Angluin's condition.

        We first give a function that achieves exhaustive generation.

\begin{lemma}
[Function for Exhaustive Generation]
\label{lem:upper-bound-exhaustive-breadth-function}
    Let $\cL$ be a countable collection of
    languages that satisfies \Cref{def:suffCondition}.
    Then, there exists a generating algorithm that, exhaustively generates from $\cL$ (and is consistent with the target language) 
    in the limit.
    The algorithm uses access to the following oracles:
    \begin{itemize}[itemsep=-1pt]
        \item[$\triangleright$] a membership oracle for 
    $\cL$,
        \item[$\triangleright$] a subset oracle for $\cL$ (that given indices $i,j$ outputs Yes if $L_i\subseteq L_j$ and No otherwise), 
        \item[$\triangleright$] a finite difference oracle for $\cL$ (that given 
    indices $i,j$ with $L_i \subset L_j$ outputs Yes if $|L_j \setminus L_i| < \infty$ and No otherwise).
    \end{itemize}
\end{lemma}
    The generation in the above result satisfies a property stronger than \cref{def:ExhaustiveGeneration}:
    \begin{remark}
        \label{rem:no-hallucination-exhaustive-generation}
        In addition to achieving exhaustive generation, the generator is consistent with the target language and, hence, does not have \textit{any} hallucinations.
    \end{remark}
    The generator in \cref{lem:upper-bound-exhaustive-breadth-function} is as follows.

  \begin{mdframed}
            {
                Fix the following: a special character $x_0 \notin \cX$ and a canonical enumeration $x_1,x_2,\dots$ of $\cX$.\\
                Initialize $\ell_0=0.$\\
            \noindent \textbf{for} $t\in \inbrace{1,2,\dots}$ \textbf{do:}
            \begin{enumerate}
                \item Observe element $x_t$ and let $S_t$ be the set of all elements observed so far.
                
                \item Construct a \textit{version space} $V_t$ consisting of all languages in $\cL_{\leq t}$ consistent with $S_t$, \ie{}, 
                \[
                    V_t\coloneqq \inbrace{L_j\colon 1\leq j\leq t\,,~~L_j\supseteq S_t}\,.
                \]

                \item If $V_t = \emptyset,$ \textbf{output} an arbitrary element of $\cX$ and \textbf{go} to the next iteration.
                
                \item[] \textit{$\#$ Define a language $L_i\in V_t$ to be \emph{critical} if $L_i$ is the smallest-indexed language in $V_t$ or $L_i$ is a subset of all languages preceding it in $V_t$, \ie{},  $L_i\subseteq L_j$ for all $1\leq j < i$.}
                
                \item Construct the set $C_t = \{L_{i_1^t} \supseteq L_{i_2^t} \supseteq \dots \supseteq L_{i^t_j}\}\subseteq V_t$ of critical languages for some $j \leq t$.
                \item[] \emph{$\#$ To construct the set of critical languages $C_t$ the algorithm needs access to the subset oracle.}
                 
                \item 
                Find the smallest indexed language $L = L(t)$ in $C_t$ such that $|L \setminus L_{i^t_j}| < \infty$. Create the set $C_t'$ by removing all the languages in $C_t$ before $L$.
                \item[] \emph{$\#$ To perform this filtering, the algorithm needs access to the finite difference oracle.}

                \item If $C_t' = \emptyset,$ \textbf{output} an arbitrary element of $\cX$ and \textbf{go} to the next iteration.

                \item Let $L_i = L_{i(t)}$ be the minimum indexed language in the set of filtered critical languages $C_t'$.

                \item If $i(t) \neq i(t-1)$, set $\ell_t = 0$; else $\ell_t = {\ell_{t-1}} + 1.$
                
                \item \textbf{output}
                the enumeration of $L_{i} \setminus \{x_0, \dots, x_{\ell_t}\}$ {induced by the canonical enumeration of $\cX$ fixed at the start.}
              
            \end{enumerate}
            }
        \end{mdframed} 
 
\begin{proof}[Proof of \Cref{lem:upper-bound-exhaustive-breadth-function}]
{We will show that the above function exhaustively generates and is consistent with the true language in the limit.
    Let $K$ be the target language and $z \in \N$ be the
    smallest number such that $L_z = K.$ 
We will use the case analysis of \Cref{lem:upper-bound-approximate-breadth}. 
Fix some symbol $x_0 \notin \cX.$}

\paragraph{Case A ($z=1$):}
Since $z=1$, the true language is the first critical language and is never filtered from $C_t'.$ Moreover, {the counters $\ell_t$ will never be reset (in Step 8) and, in fact, satisfy $\ell_t=t$.} 
{Hence, for each $t\in\N$, the algorithm $\generator_t$ enumerates the set $K \setminus (S_t \cup \{x_0,\dots,x_{t}\})$ induced by the canonical enumeration of $\cX$.
It follows that, for each removed $x_i$, there is some $t$ where it is the first element of the output enumeration.
Further, the output enumeration is always consistent with $K$.
Hence, the resulting generator exhaustively generates $K$.
In fact, it has the stronger property that it never hallucinates.}

\paragraph{Case B ($z>1$):}
Consider the languages before $L_z$ in the enumeration of $\cL$.
There are two cases: For any $i < z$, either there exists an element that belongs to $L_z$ but not $L_i$ or $L_z ~{\subseteq}~ L_i$. If the first case holds, then eventually the distinguishing element will appear in the enumeration of $K$ and make $L_i$ inconsistent. Hence, let us assume that for all $i < z$, we only care about {indices $i$ for which} $L_i \supsetneq L_z$.
{We claim that} eventually the index of Step 5 stabilizes {in the limit}.
{In particular, we will show that it stabilizes to the smallest index $i^*$ such that $L_{i^*}\supseteq L_z$ and $\abs{ L_{i^*}\setminus L_z}<\infty$; note that if there is no language $L_i\supsetneq L_z$, then $i^*$ must be $z$.}
{Before proving this claim, we show that it implies the result.}
Let {$1\leq i^*\leq z$ be} the {index} that Step 5 eventually stabilizes on. 
We know that $L_{{i^*}} \supseteq K$ {(by our earlier argument that any index $1\leq i \leq z$ not satisfying this property is eliminated after a finite time)} and $|L_{ {i^*}} \setminus K| < \infty$ {(by construction)}. 
We now show how to {exhaustively generate} $K$ in the limit, {this} corresponds to Steps 8 and 9 of the above function. 
{To see this, observe that as $|L_{{i^*}} \setminus K| < \infty$, after a finite number of steps $L_{{i^*}} \setminus \{x_0,\dots,x_{\ell_t}\} \subseteq K$ (and, hence, the algorithm eventually stops hallucinating).
    Further, since at step $t$ (for large enough $t$), we output the enumeration of $L_{i^*}\setminus\inbrace{x_0,\dots,x_{\ell_t}}$ induced by the canonical enumeration of $\cX$, it follows, for each removed $x_i$, there is some $t$ where it is the first element of the output enumeration.
    Hence, the resulting generator exhaustively generates $K$.
    In fact, it has the stronger property that it eventually stops making any hallucinations.
}

{\itparagraph{Proof of the claim.} It remains to prove our claim that the index of Step 5 stabilizes in the limit.}
Since $\cL$ satisfies the weak Angluin's condition, then $K$ has a finite tell-tale set $T_K$.
We condition on the following events: (A) $K$ {is a critical language}, and (B) $S_t \supset T_K$. 
{Condition (A) is satisfied for any $t\geq z$ and (B) is satisfied after a finite time since $T_K$ is finite and all its elements appear at a finite point in the enumeration of $K$.}
{Conditioned on these events the critical list $C_t$ is of the form}
\[
    L_{i_1^t} \supseteq L_{i_2^t} \supseteq \dots
    \supseteq K \supseteq L_{j_1^t} \supseteq \dots
\]
{First, observe that there are finitely many languages before $K$ in this list: this is because $K$ appears at a finite point in this list.
Next, we claim that conditioned on the above events the indices $i_1^t,i_2^t,\dots$ of the languages appearing \textit{before} $K$ in the list never change.
The proof is via induction.}
{\begin{itemize}
    \item \textit{Base Case:}\quad First, consider the first index $i_1^t$.
    It is defined as the smallest index language consistent with $S_t$.
    Moreover, due to the structure above it has the property that $L_{i_1^t}\supseteq K$ and, hence, it never becomes inconsistent with $S_{t'}$ for $t'\geq t$.
    Therefore, the index $i_1^t$ never changes in subsequent steps.
    \item \textit{Induction Step:} Next, we complete the induction argument, suppose indices $i_1^t,i_2^t,\dots,i_r^t$ never change in subsequent steps, then we claim that the index $i_{r+1}^t$ (if it exists) also never changes in subsequent steps.
    This is because $i_{r+1}^t$ is defined as the smallest indexed language that is (1) consistent with $S_t$ and (2) has the property that $L_{i_{r+1}^t}\subseteq L_{i_r^t}$.
    The former always holds for all subsequent $t'\geq t$ since $L_{i_{r+1}^t}\supseteq S_t\supseteq T_K$ and the latter holds for all subsequent $t'\geq t$ since $i_{r}^t$ never changes.
\end{itemize}
Now we are ready to prove that the index $i(t)$ selected in Step 5 stabilizes.
Recall that $i(t)$ is the smallest index satisfying that (1) $L_{i(t)}$ appears before $K$ in the critical list and (2) $\sabs{L_{i(t)}\setminus L_{i_j^t}} = \sabs{L_{i(t)}\setminus K} + \sabs{K\setminus L_{i_j^t}} < \infty$.
Observe that $\sabs{L_{i(t)}\setminus L_{i_j^t}} = \sabs{L_{i(t)}\setminus K} + \sabs{K\setminus L_{i_j^t}}$ and, by construction, $\sabs{K\setminus L_{i_j^t}} < \infty$ and, therefore, Condition (2) is equivalent to $\sabs{L_{i(t)}\setminus K}<\infty$.
Fix any $t$ satisfying Conditions A and B above and the corresponding $i(t)$.
For all subsequent $t'\geq t$, $L_{i(t)}$ continues to appear before $K$ in the critical list since we proved that all indices before $K$ in the critical list stabilize.
Further, $\sabs{L_{i(t)}\setminus K}<\infty$ since it is independent of $t'$.
Therefore, $i(t)=i(t')$ since $i(t)$ satisfies both properties that determine $i(t')$.
It follows that for $t'\geq t$, the index selected in Step 5 never changes.
}
\end{proof} 

\noindent Moreover, a small adaptation of the proof of \Cref{lem:upper-bound-approximate-breadth} gives a generator 
that generates exhaustively  (\cref{def:ExhaustiveGeneration}) in the limit {provided one has access to the tell-tale oracle from \cref{def:suffCondition}}.

\begin{lemma}
[Algorithm for Exhaustive Generation]
\label{rem:exhaustive-generation-upper-bound}
 Let $\cL$ be a countable collection of
    languages that satisfies \Cref{def:suffCondition}.
    Then, there exists a generating algorithm that, given access to a membership oracle for 
    $\cL$ and the tell-tale oracle from \cref{def:suffCondition}, exhaustively generates
    from $\cL$ in the limit.
\end{lemma}
\begin{proof}[Proof of \Cref{rem:exhaustive-generation-upper-bound}]
    The argument in the  proof of \Cref{lem:upper-bound-approximate-breadth} shows that the choice of the index $g_n$ stabilizes in the limit. 
    Moreover, $K \subseteq L_{g_n}$ and $\abs{L_{g_n} \setminus K} < \infty.$ 
    {To achieve exhaustive generation}, the only modification needed is that we keep track of another index $\ell_n$ which is initialized at 0, increases by 1 in every round, and every time the choice of $g_n$ changes, we reset $\ell_n = 0.$ 
    The enumeration we output is $L_{g_n} \setminus \inbrace{x_{0},\ldots,x_{\ell_n}},$ where we use the notational convention that $x_0$ is some special element that does not appear in $\cX.$ 
    Moreover, the sequence in which the element appears in the enumeration is the natural order induced by {(some canonical)} enumeration of $\cX.$ 
    Assume that $n$ is large enough so that $g_n$ has stabilized. It is easy to see two  things: for every element $\hat x$ of $L_{g_n}$, there exists some finite round $\hat n \in \N$ such that $\hat x$ is the first element in the enumeration
    we have outputted. 
    Moreover, since $L_z \subseteq L_{g_n}$ and $\abs{L_{g_n} \setminus L_z} < \infty,$ after some finite $n \in \N$ it holds that $L_{g_n} \setminus \inbrace{x_0,\ldots,x_{\ell_n}} \subseteq L_z.$ 
    Moreover, every time an element $x_i$ is omitted from {the enumeration we output}, there has been some prior iteration where it has been the first element in the enumeration.
    These arguments show that the modified generator is an exhaustive generator for $\cL.$
\end{proof}

\section{Further Results {with New Notions of Breadth and Stability}}\label{apx:stability}
    \label{apx:further}
    {In this section, we give results for language generation with new notions of breadth and stability.
    
    \paragraph{Outline.}
    In \cref{apx:further:infCoverage}, we introduce a notion of infinite coverage which weakens approximate breadth and show that it is achievable for all countable collections.
    In \cref{apx:further:infCoverage-stable}, we study generation with infinite coverage with stable generators: (1) we show that it cannot be achieved for all countable collections (\cref{apx:further:infCoverage-stable:impossiblity}), and (2) we give a sufficient condition to achieve it (\cref{apx:further:infCoverage-stable:sufficient}).
    In \cref{apx:further:incCoverage}, we present a strengthening of stability, which we call increasing coverage, and show that it can be achieved for certain collections. 
    }

    \begin{remark}[Characterizations for Existing Notions of Breadth with Stability]
        We present the characterizations of existing notions of breadth with stability in \cref{apx:landscape:stability}.
        In this section, we discuss characterizations for new notions of breadth and a strengthening of stability.
    \end{remark}

    \begin{remark}[Results allowing for Hallucinations]
        {We refer the reader to \cref{apx:landscape:breadth:hallucinations,apx:landscape:stability:hallucinations} for results on language generation with breadth when some amount of hallucination is allowed.}
    \end{remark}
    
    \subsection{Generation with Infinite Coverage}\label{apx:further:infCoverage}
    In this section, we provide further motivation behind \Cref{def:approxBreadth}{, generation with approximate breadth}.
    An immediate modification of the algorithm of \citep{kleinberg2024language}
        can achieve \emph{finite coverage} of the target language, for any 
        finite number. More concretely, for any function $f\colon\N \rightarrow \N$
        and any countable collection of languages $\cL$
        there exists a generating algorithm $\inparen{\generator_n}_{n \in \N}$
        such that, for any target language $K \in \cL$ and any enumeration of $K$
        the algorithm achieves in the limit
        \[
            \supp(\generator_n)\subseteq K\,,\qquad 
            \supp(\generator_n) \cap S_n = \emptyset\,,
            \qquadand
            \abs{\supp(\generator_n)} = f(n)\,,
        \]
        where $S_n$ is the set of elements enumerated until round $n.$
        In fact, their algorithm can achieve the stronger property of \emph{infinite coverage} defined below.

        \begin{definition}
        [Language Generation with Infinite Coverage in the Limit]\label{def:infinite-coverage}
            A generating algorithm $\generator=(\generator_n)$ is said to generate with infinite coverage in the limit for a language collection $\cL=\inbrace{L_1, L_2,\dots}$ if, for any $K\in \cL$ and enumeration of $K$, there is an $n^* \geq 1$, such that {for all $n \geq n^*$,} after seeing {$n$} elements of the enumeration (corresponding to the set $S_n$ in round $n)$, 
                    \begin{equation*}\supp(\generator_n)\subseteq K\,,\quad 
            \supp(\generator_n) \cap S_n = \emptyset\,,
            \quadand
            \abs{\supp(\generator_n)} = \infty\,,
            \end{equation*}
        \end{definition}
       Given the above notion of infinite coverage, a simple modification to the generating algorithm of \citep{kleinberg2024language} gives the following result.

         \begin{proposition}
    [Modification of  \citep{kleinberg2024language}]\label{lem:upper-bound-infinite-coverage}
    There is a generating algorithm with the property that for any countable collection of languages $\cL =
    \{L_1, L_2,\dots\}$, any target language $K \in \cL,$ and any enumeration of $K$, the algorithm generates with infinite coverage from $K$ in the limit.
    \label{thm:coverage}
     \end{proposition}
    Thus, the aforementioned modification of the algorithm of \citep{kleinberg2024language} has the property that it does not hallucinate (\ie{}, it does not include any elements outside of $K$ in its support) and covers infinitely  many (unseen) elements of the target language,
    but might, potentially, not cover infinitely many elements as well. 
    Thus, a natural question is whether there exists an algorithm that does not hallucinate, can cover infinitely many elements of $K,$ and also miss only \emph{finitely} many elements of it. 
    This is precisely the requirement of generation with approximate breadth (\Cref{def:approxBreadth}). 

     \begin{proof}[Proof Sketch of \Cref{thm:coverage}]
    We discuss a sketch 
       of the proof for the version of the 
       algorithm of \citep{kleinberg2024language} that uses a subset oracle for $\cL,$ \ie{}, for any $L_i, L_j \in \cL$ it can ask ``Is $L_i \subseteq L_j?$''.
       Let us first give a high-level description of their algorithm.
       For large enough $n \in \N,$ it creates a (potentially infinite) sequence of languages $\cL' = \inbrace{L_{i_1},L_{i_2}, \ldots} \subseteq \cL$ such that the following hold.
       \begin{itemize}
           \item[\textbf{(i)}] For every language $L \in \cL'$ it holds that $L$ is consistent, \ie{}, $S_n \subseteq L,$ where $S_n$ is the set of elements enumerated until round $n$,
           \item[\textbf{(ii)}] The sequence of languages in $\cL'$ satisfies the inclusion: $L_{i_1} ~{\supseteq}~ L_{i_2} ~\supseteq \ldots,$ and 
           \item[\textbf{(iii)}] $K \in \cL'.$
       \end{itemize}
       Then, it outputs an arbitrary string $x$
       such that $x \notin S_n$ and $x \in L_{i_\ell},$
       where $i_{\ell} \in \N$ is the largest number
       such that $L_{i_\ell} \in \cL'$ and $i_\ell \leq n.$ 
        The immediate modification
       is to output a distribution $\generator_n$
       such that $\supp(\generator_n) = L_{i_\ell}\setminus S_n.$
       Notice that this can
       be done in a computable way: in order to sample
       from this distribution, we first sample a
       natural number $\hat n$ (\eg{}, from a geometric distribution on $\N$), and then we check if $x_{\hat n} \in L_{i_\ell} \setminus S_n.$
 \end{proof}
    An analogous modification can be made to the algorithm of \citep{kleinberg2024language} that only has access to a membership oracle for $\cL.$ For brevity, we omit the modifications to this algorithm.
  
     \begin{remark}
    [Oracle Access for Results in \Cref{fig:characterization}] 
    {Following the phrasing of \citep{kleinberg2024language}, we provide both \emph{functions} and \emph{algorithms} that generate in the limit. An algorithm only accesses $\cL$ via a membership oracle (and potentially a tell-tale oracle). When a generator uses other types of oracles (\eg{}, subset oracle), we call it a \textit{function}. }
     \end{remark}

    \subsection{Infinite Coverage with Stable Generators}
        \label{apx:further:infCoverage-stable}
        {In this section, we continue the study of infinite coverage, exploring when it can be achieved with stable generators.}
    \subsubsection{A Collection for which No Stable Generator Has Infinite Coverage}
        \label{apx:further:infCoverage-stable:impossiblity}
    In this section, we show that there is a language collection $\cL$ for which there exists an algorithm that achieves approximate breadth in the limit, but no stable algorithm can achieve the (strictly) weaker notion of generating with infinite coverage in the limit.  
The collection $\cL$ is due to \citep{charikar2024exploringfacetslanguagegeneration},  who observed that a trivial generating algorithm that does not get \emph{any} input generates from $\cL$ exhaustively in the limit. 
Since exhaustive generation implies, by definition, generation with approximate breadth, we only need to prove the impossibility result for generation with infinite coverage by stable generators.
 
We first provide the collection and then state the result.

\begin{example}[\citep{charikar2024exploringfacetslanguagegeneration}]\label{ex:weak-Angluin-no-angluin}
     Let $\cX = \N, L_\infty = \N,$  for every $i \in \N$ let $L_i = \N \setminus \inbrace{i},$ and let $\cL = \inbrace{L_\infty, L_1,L_2,\ldots}.$ Notice that
    every pair of languages $L_i, L_j \in \cL$ differ in at most two elements, so 
    it follows that $\cL$ satisfies \cref{def:suffCondition}.
    To see that it does not satisfy Angluin's condition (\cref{def:angliun-criterion}), consider 
    the language $L_{\infty}.$ Then, for every finite subset $T \subseteq L_{\infty}$
    there is some language $L_T$ such that $T \subseteq L_T$ and
    $L_T \subsetneq L_\infty.$
\end{example}
We continue with the statement of the theorem.

\begin{theorem}\label{thm:impossibility-stable-generation}
    There exists a countable collection of languages $\cL$ that 
    satisfies the weak Angluin's condition (\Cref{def:suffCondition}), and 
    for which no stable generating algorithm can achieve generation
    with infinite coverage in the limit (\cref{def:infinite-coverage}).
\end{theorem}

\begin{proof}
    {Consider the collection defined in \Cref{ex:weak-Angluin-no-angluin}. Since it satisfies
    the weak Angluin's condition
    (\cref{def:weakAngluin}),} by 
   {\cref{thm:approximate-breadth}}, it follows that 
    there exists an algorithm that achieves generation with approximate breadth 
    in the limit.\footnote{As we explained, this also follows {from the work of} \citep{charikar2024exploringfacetslanguagegeneration}.} Assume
    towards contradiction that there exists a stable 
     generating algorithm $\generator = \inparen{\generator_n}_{n \in \N}$
     that achieves generation with infinite coverage in the limit.
     We will pick a target language and an enumeration of it that witnesses
     the lower bound based on the given algorithm $\generator.$
     We denote the target language by $K$ and the target enumeration by $E_K^\infty.$
     Like in the previous proofs, for any enumeration $E$, we use the notation
        $E(i)$ to denote its $i$-th element, $E(1:i)$ to denote its first $i$
        elements, and $E(i:\infty)$ to denote all but the first $i-1$ elements.

    As in the previous proofs of the impossibility results, we consider
    several phases for our construction. 
    First, we start with the enumeration $E_\N^\infty = (1,2,3,\ldots).$
    Notice that this is a valid enumeration for $L_\infty.$
    We consider two cases: \textbf{(I)} either there is some $n \in \N$ such that 
    ${\abs{\supp(\generator_{n})}} = \infty,$ or  \textbf{(II)} if there is no such $n$ the 
    lower bound follows immediately by picking $K = \N$ and the hard enumeration
    $E_K^\infty = E_\N^\infty.$ For the continuation of the proof, assume that the former 
    case holds and let $n_1$ denote the first timestep for which this holds. Notice
    that up to that point we have enumerated $(1,\ldots,n_1).$ 
    Let $\hat n_1 \in \N$ be the smallest number strictly greater than $n_1$ that
    is in the support of $\generator_{n_1}.$ Notice that such a number must exist
    because $\abs{\supp(G_{n_1})} = \infty.$ 
    
    We now extend the target enumeration
    $E_K^\infty(1:\hat n_1 - 1) = (1,2,\ldots,\hat n_1 - 1).$ Notice that this is well-defined since we only add elements to the already constructed enumeration.
    We continue building the target enumeration by skipping the element $\hat n_1$ 
    and including the element $\hat n_1 + 1$ to it, \ie{}, the $\hat n_1$-th element
    of the constructed enumeration is $\hat n_1 + 1.$
    We continue adding consecutive elements to the enumeration $E_K^\infty$ until
    the first timestep $n > \hat n_1 + 1$ such that $\supp\sinparen{\generator_n} \neq 
    \supp\sinparen{\generator_{n_1}}$ and $\abs{\supp(\generator_n)} = \infty.$
    Notice that if no such $n$ exists the lower bound already follows
    by picking the target language $K = L_{\hat n_1}$ and the constructed
    target enumeration. This is because in every timestep either $\supp\sinparen{\generator_n} = \supp\sinparen{\generator_{n_1}}$ (and therefore
    $\supp(\generator_n) \not\subseteq K$ because $\hat n_1 \in \supp(\generator_n)$) or
    $\abs{\supp(\generator_n)} < \infty,$ hence the algorithm does not achieve 
    generation with infinite coverage in the limit.
    For the continuation of the proof, 
    let $n_2$ denote the first
    timestep for which $\supp\sinparen{\generator_{n_2}} \neq 
    \supp\sinparen{\generator_{n_1}}$ and $\abs{\supp(\generator_{n_2})} = \infty.$
    We then add the element $\hat n_1$ to the constructed prefix of the enumeration $E_K^{\infty}$ and terminate the first phase.

    Notice that at the end of the first phase we have enumerated all the elements
    $\inbrace{1,2,\ldots,n_2-1}$ and the support of the generating algorithm
    has changed at least once or we have the desired lower bound. We continue inductively in exactly the same way 
    until \textbf{(I)} either some phase cannot be terminated in which case the lower
    bound follows because the property of infinite coverage in the limit is not
    achieved or \textbf{(II)} we construct infinitely many phases which 
    witness infinitely many changes in the support of the generating algorithm, hence
    showing it cannot be stable. This concludes the proof.   
\end{proof}

    \subsubsection{Sufficient Condition for Stable Generation with Infinite Coverage}
        \label{apx:further:infCoverage-stable:sufficient}
        
    In this section, we provide a sufficient condition on the language collection $\cL$ that guarantees the existence of a stable generating algorithm that generates with infinite coverage in the limit. 
    In particular,  we can show that if a collection has finite closure dimension \citep{li2024generationlenslearningtheory},  then there exists a stable generating algorithm that achieves infinite coverage in the limit. 
    First, we give the definition of the closure dimension \citep{li2024generationlenslearningtheory}, which is inspired by a result of \citep{kleinberg2024language} on \emph{uniform generation}\footnote{The exact definition of uniform generation is not important for our work. 
    At a high level, this condition asks whether there exists some $d \in \N$ such that after the generator observes $d$ different strings from \emph{any} target language of $\cL$, then  it can generate unseen strings that belong to $K$.} from finite sets of languages.

\begin{definition}[Closure Dimension \citep{li2024generationlenslearningtheory}]\label{def:closure-dim}
    The closure dimension of $\cL,$ denoted by $\mathrm{d}(\cL),$ is the largest natural
    number $\ell \in \N$ for which there exist distinct $x_1,\ldots,x_\ell \in \cX$ such that 
    \[
        {V\sinparen{x_1,\ldots,x_{\ell}}}\coloneqq \inbrace{L \in \cL\colon \inbrace{x_1,\ldots,x_\ell} \subseteq L} \neq \emptyset \quadand 
        \abs{\bigcap_{L \in {{V\sinparen{x_1,\ldots,x_{\ell}}}}} L} < \infty \,.
    \]
    If for every $\ell \in \N$ there exists a set of distinct elements
    that satisfies this condition
    we say that $\mathrm{d}(\cL) = \infty.$
\end{definition}
\noindent {In general the closure dimension can be $\infty$, but due to a result of \citep{kleinberg2024language}, we know that all collections of languages with finitely many languages have finite closure dimension.} 
In order to design an algorithm that achieves stable infinite coverage
for any collection $\cL$ that has a finite closure dimension, we will make use
of a stronger oracle for $\cL$ than just the membership oracle to it.
Namely, we define the \emph{version space intersection} (VSI) 
membership oracle as follows. 
\begin{definition}[Membership Oracle to Version Space Intersection (VSI)]
    \label{def:membership-vsi}
    The membership oracle to VSI is a primitive that, given a set of distinct elements $x_1,\ldots,x_n\in \cX$ and a target element $x\in \cX$, returns
    \[
        \ind\inbrace{
            x \in \cap_{L \in {V\inparen{x_1,\ldots,x_n}}} L
        }
        \,.
    \]
\end{definition}
\noindent We remark that for finite collections $\cL$ this oracle can be computed 
just with membership oracle to $\cL,$ but for countable
collections this oracle might not be computable. 

\begin{proposition}[Adaptation of Lemma 3.2 in  \citep{li2024generationlenslearningtheory}]\label{thm:finite-closure-implies-stable-inf-coverage}
    Let $\cL$ be a collection of languages with $\mathrm{d}(\cL) < \infty$ (\cref{def:closure-dim}). There exists a stable (\cref{def:stable-generators}) generating algorithm $\generator=\inparen{\generator_n}$ for $\cL$ that, given the value of $\d(\cL)$, achieves infinite coverage (\cref{def:infinite-coverage}) using access to a VSI
    membership oracle for $\cL$, after taking as input $\mathrm{d}(\cL)+1$
    distinct {elements}.
\end{proposition}
In particular, since the closure dimension of any finite collection of languages is finite \citep{kleinberg2024language}, for any finite collection of languages, there exists a stable generating algorithm that achieves infinite coverage. It is not hard to see that for such collections, the VSI oracle can be implemented using only membership oracle to languages in $\cL.$
\begin{corollary}[Stable Generation for Finite Collections]\label{cor:stability-finite}
    For every finite collection of languages $\cL$, the following hold:
    \begin{enumerate}
        \item There exists a stable generating algorithm that achieves generation
        with exact breadth in the limit, using only membership oracle access to 
        $\cL.$

        \item There exists a stable generating algorithm that achieves generation
        with infinite coverage after taking as input $\mathrm{d}(\cL)+1$
        distinct strings, using only membership oracle access to 
        $\cL.$
    \end{enumerate}
\end{corollary}
Moreover, {for finite collections, a stronger property is possible:} the results of \citep{kalavasis2025limitslanguagegenerationtradeoffs} {(see Proposition 3.9 in their work)} show that for finite collections there exists a stable generating algorithm that achieves exact breadth in the limit (and, hence, also infinite coverage), but there might not be {an upper} bound on the {elements} needed to achieve this property.\footnote{To be precise, Proposition 3.9 in \citep{kalavasis2025limitslanguagegenerationtradeoffs} gives an algorithm to identify finite collections in the limit. This algorithm immediately gives an algorithm for generation with exact breadth: once we know an index $z$ such that $K=L_z$, we can sample a natural number (from, \eg{}, an exponential distribution on $\N$) and output the $i$-th element of $L_z$. The latter, in turn, can be found using the membership oracle to $L_z$.}

\smallskip 

Finally, we prove \cref{thm:finite-closure-implies-stable-inf-coverage}.
\begin{proof}[Proof of \cref{thm:finite-closure-implies-stable-inf-coverage}]
    {Our proof is inspired by the Lemma 3.2 from \citep{li2024generationlenslearningtheory}.} 
    {The only modification
    is that now the algorithm stops using
    new elements beyond the $\mathrm{d}(\cL)+1$ 
    elements required to achieve infinite coverage.
    Moreover, we discuss the type of access
    to $\cL$ needed that is sufficient to achieve this property, which
    was not the focus of \cite{li2024generationlenslearningtheory}.
    }
    Let $K \in \cL$ be any target language and $x_1,\ldots,x_{\mathrm{d}(\cL)+1} \in K$
    be any $\mathrm{d}(\cL)+1$ distinct elements of the target language.
    First, notice that since $x_1,\ldots,x_{\mathrm{d}(\cL)+1} \in K$,
    $V\sinparen{x_1,\ldots,x_{\mathrm{d}(\cL)+1}} \neq \emptyset,$ as $K \in V\sinparen{x_1,\ldots,x_{\mathrm{d}(\cL)+1}}.$
    By the definition of the closure dimension (\cref{def:closure-dim}) {and since $\abs{K}=\infty$ (recall that language generation is not meaningful with finite languages and, hence, throughout this work, we consider all languages are infinite)},
    \[
        \abs{\bigcap_{L \in V\sinparen{x_1,\ldots,x_{\mathrm{d}(\cL)+1}}} L } = \infty \qquadand \bigcap_{L \in V\sinparen{x_1,\ldots,x_{\mathrm{d}(\cL)+1}}} L \subseteq K \,.
    \]
    Thus, the generating algorithm can stabilize its support to be
    {$T\coloneqq \bigcap_{L \in V\sinparen{x_1,\ldots,x_{\mathrm{d}(\cL)+1}}} L$} and never
    change it from this point on during the interaction with the adversary.
    Notice that given access to a VSI membership oracle for 
    $\cL$ the learner can indeed sample from a distribution supported on {$T$} as follows: first
    sample a natural number $\hat n$ {(\eg{}, from a geometric distribution on $\N$)} and then query the VSI membership oracle with 
    the set of elements $x_1,\ldots,x_{\mathrm{d}(\cL)+1}$ and the target
    element $x_{\hat n}.$\footnote{To be formal, we need to use a different enumeration of the strings of $\cX$ and the strings that define the target version space. We overload the notation {for simplicity}.} Repeat the process until the oracle {returns Yes}.
    Notice that this process terminates with probability 1, and the support of
    the induced distribution is exactly {$T$}.
\end{proof}

\noindent As a final note on our discussion on stability,
it is worth pointing out that there are collections
that do not satisfy the weak Angluin's condition, nevertheless
there is a stable generating algorithm
that achieves infinite coverage after observing one
example from the target language. The example
is due to \cite{charikar2024exploringfacetslanguagegeneration}.

\begin{example}[Stable Infinite Coverage $\centernot\implies$ Weak Angluin's Condition]\label{ex:ar-progression}
    Define the domain $\cX$ and the language collection $\cL$ as follows 
    \[
        \cX = \Z 
        \qquadand
        \cL = \inbrace{L_{\infty} \coloneqq \Z, L_{a} \coloneqq \inbrace{a+i, i \in \N}\colon a \in \Z}\,,
    \]
    where $\Z$ is the set of integer numbers. 
    Notice that both $\cX$ and $\cL$ are countable, and each $L \in \cL$ is also countable.
    Consider the language $L_\infty$
    and any finite $T \subseteq L_{\infty}.$ Let $i_T$ be the smallest
    element of the subset $T.$ Then, $T \subseteq L_{i_T}, L_{i_T} \subsetneq L_\infty,$ and $\abs{L_\infty \setminus L_{i_T}} = \infty.$ Hence, this
    collection does not satisfy the weak Angluin's condition. 
    Consider the generating algorithm $\generator$ which in every round
    $n$ outputs a distribution with $\supp\inparen{\generator_n} = \N \setminus S_1,$ where $S_1$ is the input in round 1.
    It is not hard to see that for any target
    language $K$, this generating algorithm achieves infinite coverage, and
    is, by definition, stable.
\end{example}

\subsection{Generation with Increasing Coverage: A Strengthening of Stability}
    \label{apx:further:incCoverage}
    {In this section, we introduce new property of generation -- increasing coverage, which is a strengthening of stable generation.}
    
    A key observation in \cite{kleinberg2024language} is that their generator's support can decrease when it sees new strings from the target $K$ and, in fact, for many language collections the number of valid strings omitted from its support can grow without bound, which is an extreme form of \textit{mode collapse}.
    In this light, one can view stability as a property that avoids such extreme mode collapse: any stable generator can only change its support finitely many times.
    A natural question is whether we can achieve something stronger than stability and, yet, more tractable than breadth.
    To capture this phenomenon, we introduce the following notion of \emph{generation with strictly increasing coverage}.

\begin{definition}[Generation with Strictly Increasing Coverage]\label{def:strictly-increasing-coverage}
    Let $\cL$ be a countable collection of languages. 
    A generating algorithm $\generator = \inparen{\generator_n}$ is said to have strictly increasing coverage for $\cL$ in the limit if, for any $K \in \cL$ and enumeration of $K$, there is an $n^* \geq 1$ such that {for all $n \geq n^*$,} after seeing {$n$} elements of the enumeration, the following hold
    \begin{itemize}
        \item $\supp\inparen{\generator_n} \subseteq \supp\inparen{\generator_{n+1}},$ and

        \item either $\supp\inparen{\generator_{n}}  = K$ or
        there exists some $n' > n$ such that $\supp\inparen{\generator_n} \subsetneq \supp\inparen{\generator_{n'}}.$
    \end{itemize}
\end{definition}
    Intuitively, if a generator satisfies this property of strictly increasing coverage, then, at a high level, one may gather that it learns something new about the target language each time it sees a new string from it.

    To gain intuition about when increasing coverage is achievable, let us consider two extremes.
    On the one hand, it is not hard to see that achieving approximate breadth along with strictly increasing coverage is significantly harder than achieving approximate breadth along:
    This is because if a generator has approximate breadth, then after seeing sufficiently many strings from $K$, its support only misses a finite number of strings from $K$ and, then, if it further has strictly increasing coverage, its support eventually becomes equal to $K$ implying exact breadth which is only achievable for collections satisfying Angluin's condition ({\cref{thm:exact-breadth}}).
    On the other hand, if one is not required to have infinite coverage\footnote{For the subsequent discussion, we use the equivalent version of the definition of infinite coverage (\cref{def:infinite-coverage}) which allows
    the support of the generator to contain strings from the
    set $S_n$, which is the set of all strings enumerated so far.} (a requirement already weaker than any notion of breadth), then it is easy to achieve strictly increasing coverage: 
        consider the generator $\generator$ in \cref{lem:upper-bound-infinite-coverage}, which achieves infinite coverage for any collection $\cL$, and post-process the algorithm to have a support of size at most $t$ on round $t.$
        Since eventually $\generator$'s support has infinitely many elements (as it achieves infinite coverage), it follows that the support of the above post-processed variant increases infinitely many times, implying that the post-processed variant achieves strictly increasing coverage.

    Thus, the most interesting question is whether there is a generator that achieves infinite coverage -- a property between breadth and consistent generation --  while also having strictly increasing coverage.
    Our next result shows that there are collections for which this is indeed possible. 
    The collection we use to show this result does not satisfy the weak Angluin's condition, so one cannot achieve even the weakest notion of breadth (namely, approximate breadth or equivalently exhaustive generation) for this collection.

\begin{proposition}
    There exists a countable collection of languages $\cL$ that does not satisfy the weak Angluin's condition (\cref{def:weakAngluin}) and for which there exists a generating algorithm $\generator = \inparen{\generator_n}$ that can achieve infinite coverage (\cref{def:infinite-coverage}) and has strictly increasing coverage in the limit (\cref{def:strictly-increasing-coverage}).
\end{proposition}

\begin{proof}
    Consider the collection of arithmetic progressions used in 
    \cref{ex:ar-progression}. As we discussed, this collection does 
    not satisfy the weak Angluin's condition. 
    Let $S_n$ be the set of elements enumerated up to round $n$
    and let $\hat t_n$ denote the smallest element of $S_n.$
    Then, it is immediate that the generating algorithm that outputs a distribution supported on $\inbrace{\hat t_n, \hat t_n +1, \ldots}$ achieves infinite coverage and has strictly increasing coverage in the limit.
\end{proof}
\noindent {We remark} that the generating strategy 
in the above result uses information about the
structure of $\cL,$ and not just membership access to it.

\negspacePre{}\negparaspace{}
\section{Concluding Remarks}\label{sec:conclusion}
\negspacePost{}
\negspacePost{} 
{In this work, we continue the study of language generation, a nascent area introduced by \citet{kleinberg2024language}. On a conceptual level, our results -- building on prior work -- offer a resolution to the main open question of Kleinberg and Mullainathan showing that, indeed, a tension between validity and breadth is inherent in language generation, at least under all the formal notions of breadth considered in prior work \citep{kalavasis2025limitslanguagegenerationtradeoffs,charikar2024exploringfacetslanguagegeneration}.
On a technical level, we introduce  novel diagonalization-based lower bound techniques and new algorithms that achieve generation with breadth whenever possible.  Though we focus on the prompt-less setting, our techniques extend to the prompted generation setting as well \citep{kleinberg2024language}.
Our work suggests several promising directions for future work: investigating weaker notions of breadth, completing the characterizations for certain novel variants of stable generation (\cref{apx:landscape:stability}), and identifying what additional information beyond positive examples could help generators achieve both validity and breadth -- an intriguing challenge given our impossibility results.}

    \section*{Acknowledgments}
        We thank Moses Charikar and Chirag Pabbaraju, the authors of \citep{charikar2024exploringfacetslanguagegeneration}, for coordinating the arXiv submissions of their updated work \citep{charikar2024exploringfacetslanguagegenerationV2} and this work.
        We thank Jon Kleinberg for a discussion regarding the representation of the generators.
        We thank Manolis Zampetakis for feedback on a draft of this paper.
        Alkis Kalavasis was supported by the Institute for Foundations of Data Science at Yale.
        Grigoris Velegkas was supported by the AI Institute for Learning-Enabled Optimization at Scale (TILOS).
        
\newpage
 
\appendix
\printbibliography
\newpage

\section{Additional Remarks and Discussion}
    {In this section, we present additional remarks and discussions.}

    \subsection{Overview of Kleinberg and Mullainathan's Algorithm}
    \label{apx:km-algo}
         In this section, we give a high-level description of the algorithm of \citet{kleinberg2024language}. 
         Consider some fixed language collection $\cL = \inbrace{L_1, L_2, \ldots}$. 
         Now consider any enumeration the 
         adversary gives as input to the generator.
        In every round $n \in \N,$ the generation algorithm of \citet{kleinberg2024language} creates a (potentially infinite) sequence of languages $\cL' = \inbrace{L_{i_1},L_{i_2}, \ldots} \subseteq \cL$ such that the following holds:
       \begin{itemize}
           \item[\textbf{(i)}] For every language $L \in \cL'$ it holds that $L$ is consistent, \ie{}, $S_n \subseteq L,$ where $S_n$ is the set of elements enumerated until round $n$,
           \item[\textbf{(ii)}] For every language
           $L_{i_j} \in \cL'$ it holds that $L_{i_j} \subseteq L_{i_{j'}}, \forall j' \leq j.$
       \end{itemize}
       Then, it outputs an arbitrary string $x$
       such that $x \notin S_n$ and $x \in L_{i_\ell},$
       where $i_{\ell} \in \N$ is the largest number
       such that $L_{i_\ell} \in \cL'$ and $i_\ell \leq n.$ 
       The main ingredient of the proof is that for all 
       $n$ sufficiently large the target language $K$ will be part
       of $\cL'.$ Moreover, languages that come after it 
       are subsets of $K.$ Thus, it is safe to be generating elements
       from these languages.

    \subsection{{On Uniqueness of Unambiguous Generation and Exhaustive Generation}}
        {In this section, we show that unambiguous generation satisfies the uniqueness criterion and that exhaustive generation satisfies the finite non-uniqueness criterion.}

        \paragraph{Unambiguous Generation Satisfies Uniqueness.}
        To see {the former}, consider any  distinct languages $L\neq L'$.
        Suppose a generator $\generator$ unambiguously generates from $L$.
        This implies that 
        \[
            \abs{
                \supp(\generator) \triangle L
            } < \min_{L''\in \cL,~ L'' \neq L} \abs{
                \supp(\generator) \triangle L''
            }\,.
        \]
        However, setting $L''=L'$ implies that $\abs{
                \supp(\generator) \triangle L
            } < 
            \abs{
                \supp(\generator) \triangle L'
            }$ which shows that $\generator$ does not unambiguously generate from $L'$.
        This proves the following result.
        \begin{observation}
            \label{obs:unambiguous-unique}
            Unambiguous generation (\cref{def:unambiguous}) satisfies the uniqueness criterion. %
        \end{observation}

        \paragraph{Exhaustive Generation Satisfies Finite Non-Uniqueness.}
        Recall that in the formulation of exhaustive generation, the generating algorithm is a sequence of mappings from sequences of the domain to \emph{enumerations} of the domain. 
        Let $\generator(1:\infty)$ be the set containing all the items $\generator$ enumerates.
        To see the claim, consider any pair of languages $L$ and $L'$ that differ in infinitely many elements, \ie{}, $\abs{L\triangle L'}=\infty$.
            Now, if a generator $\generator$ generates exhaustively generates both $L$ and $L'$, then, by definition
            \[
                \abs{
                    L \setminus \generator(1:\infty)
                }\,,\quad
                \abs{
                    L' \setminus \generator(1:\infty)
                }\,,\quad
                \abs{
                    \generator(1:\infty)\setminus L
                }\,,\quad
                \abs{
                    \generator(1:\infty)\setminus L'
                }
                ~~<~~
                \infty\,.
                \yesnum\label{eq:finiteNonUniqueness:EG}
            \]
            This contradicts the fact that $\abs{L\triangle L'}=\infty$ since 
            \begin{align*}
                \abs{L \triangle L'} 
                    &~~=~~ 
                    \abs{L\setminus L'} + 
                    \abs{L'\setminus L}\\
                    &~~\leq~~
                        \inparen{
                            \abs{(1:\infty)\triangle L'} + 
                            \abs{L\setminus \generator(1:\infty)}
                        } + 
                        \inparen{
                            \abs{\generator(1:\infty) \triangle L} + 
                            \abs{L'\setminus \generator(1:\infty)}
                        }\\
                    &~~\leq~~3 \cdot  \inparen{
                        \abs{L\setminus \generator(1:\infty)}
                        + \abs{\generator(1:\infty)\setminus L}
                        + \abs{L'\setminus \generator(1:\infty)}
                        + \abs{\generator(1:\infty)\setminus L'}
                    }\\
                    &~~\Stackrel{\eqref{eq:finiteNonUniqueness:EG}}{<}~~ \infty\,.
            \end{align*}
        \begin{observation}
            \label{obs:exhautive-generation-finite-non-unique}
            Exhaustive generation (\cref{def:ExhaustiveGeneration}) satisfies the finite non-uniqueness criterion.
        \end{observation}
        \begin{remark}[Exhaustive Generation Does Not Satisfy the Uniqueness Criterion]
            Note that the above proof can be made constructing -- there is a generator which generates exhaustively from both $L$ and $L'$ provided $L$ and $L'$ differ in finitely many elements.
            This implies that exhaustive generation does not satisfy the uniqueness criteria.
        \end{remark}

    \subsection{Membership Oracle Problem}\label{apx:mop}
        
        In this section, we define the Membership Oracle Problem (MOP), which is required for the impossibility results of \citep{kalavasis2025limitslanguagegenerationtradeoffs}, but not required for the characterizations in our work. For more details, we refer to Definitions 5 and 6 in \citep{kalavasis2025limitslanguagegenerationtradeoffs}.
        
        \begin{restatable}[Membership Oracle Problem \citep{kalavasis2025limitslanguagegenerationtradeoffs}]{definition}{defMOP}\label{def:membership access}\label{def:mop}
            Given a generator $\generator$,
            the membership oracle problem for $\generator$, denoted as $\mathsf{MOP}(\generator)$, is defined as follows: given the description of $\generator$ and a string $x$, output $\textsf{Yes}$ if $x\in \supp(\generator)$ and output $\textsf{No}$ otherwise.
        \end{restatable}

\section{{Formal Definition of} Language Identification {in the Limit}}\label{apx:further-background}
\label{sec:furtherbackground}   
    In this section, {we provide the formal definition of language identification in the limit.} %
        
    For a fixed collection $\cL$, an adversary and an identifier play the following game: 
    The adversary chooses a language $K$ from $\cL$ without revealing it to the identifier, and it begins \emph{enumerating} the strings of $K$ (potentially with repetitions) $x_1,x_2,\dots$ over a sequence of time steps $n = 1,2,3,\dots$. 
    The adversary can repeat strings in its enumeration,
    but the crucial point is that for every string $x \in K$, there must be at least one time step $n$ at which
    it appears. At each time $n$, the identification algorithm $\algo{I}$, given the previous examples $x_1,x_2,\dots,x_n$, outputs an index $i_n$ that corresponds to its guess for the index of the true language $K$.
        Language identification in the limit is then defined as follows.
        \begin{definition}
            [Language Identification in the Limit \citep{gold1967language}]\label{def:Identification}
            Fix some $K$ from the language collection $\cL = \{L_1, L_2,\dots\}$.
                The identification algorithm  $\algo{I} = (\algo{I}_n)$ identifies $K$ in the limit if there is some $n^* \in \N$ such that for all steps $n > n^*$, the identifier’s guess $i_n$ satisfies $i_{n} = i_{n-1}$ and $L_{i_n} = K.$
                The language collection $\cL$ is identifiable in the limit if there is an identifier that identifies in the limit any $K \in \cL,$ for any enumeration of $K$.
                {In this case, we say that the identifier identifies the collection $\cL$ in the limit.}
        \end{definition}
        It is important to note that the above definition imposes some stability to the algorithm: since there can be multiple appearances of $K$ in the enumeration of $\cL$, an algorithm identifies $K$ in the limit only if it eventually \emph{stabilizes} (\ie{}, $i_n = i_{n-1}$ for $n$ larger than some $n^*$) to a correct index (\ie{}, $L_{i_n} = K$).
        A natural question is which collections of languages are identifiable in the limit. Angluin \citep{angluin1980inductive} provided a condition that characterizes such collections (see \cref{def:angliun-criterion}). %
        \begin{theorem}
        [Characterization of Identification in the Limit \citep{angluin1980inductive}]
        \label{thm:angluin-id-limit}
        The following holds for any countable collection of languages $\cL.$
        \begin{enumerate}
            \item $\cL$ is identifiable in the limit if it satisfies Angluin's condition and one has access to the tell-tale oracle.
            \item If there is an algorithm that identifies $\cL$ in the limit, then Angluin's condition is true and the tell-tale oracle can be implemented.
        \end{enumerate}
        \end{theorem}
        The above tight characterization shows that language identification is information-theoretically impossible even for simple collections of languages, such as the collection of all regular languages. 
        {Crucially, access to the tell-tale oracle is necessary for identification in the limit (its existence alone is not sufficient) \cite[Theorem 2]{angluin1980inductive}.}

\section{{Comparison to \citet{charikar2024exploringfacetslanguagegeneration,charikar2024exploringfacetslanguagegenerationV2}}}\label{appendix:relatedWorks}
            See \cref{sec:related-work} for a timeline of the works \citet{charikar2024exploringfacetslanguagegeneration}, \citet{charikar2024exploringfacetslanguagegenerationV2}, and the present work.
            In the following, we map the relevant results of \citet{charikar2024exploringfacetslanguagegenerationV2} to some of our results.
            \begin{itemize}
                \item \textbf{Characterization of Generation with Exact Breadth:}
                    Their result showing that Weak Angluin’s Condition with Existence (Proposition 6.1 in their work) is necessary for exhaustive generation is comparable to the lower bound for exhaustive generation in \cref{Main}.
                    Their result showing the sufficiency of Weak Angluin’s Condition with Existence (Proposition 6.2 in their work) for exhaustive generation is comparable to the upper bound for exhaustive generation in \cref{lem:upper-bound-exhaustive-breadth-function}.
                    Their result showing the sufficiency of Weak Angluin’s Condition with Enumeration (Proposition 6.2 in their work) for exhaustive generation with only membership queries is comparable to \cref{rem:exhaustive-generation-upper-bound}.

                \item \textbf{Characterization of Exhaustive Generation:}
                    Their result showing that Angluin’s Condition is necessary for generation with exact breadth (Proposition 5.3 in their work) is comparable to the upper bound in \cref{thm:exact-breadth}.
            \end{itemize}
            Finally, as mentioned in \cref{sec:related-work}, our work provides several additional contributions for existing notions of breadth/stability beyond these shared results (see \cref{sec:results:characterizations,sec:results:characterizations,rem:beyondBreadth:1,rem:beyondBreadth:2,sec:results:stability,rem:statistical}). 
            Further, our work also introduces new notions of breadth/stability and provides results for them (see \cref{apx:landscape:breadth:hallucinations,apx:landscape:stability:hallucinations,apx:further}).

\newpage

\end{document}